\documentclass[11pt]{article}

\usepackage{fullpage}
\usepackage{amsmath,amsthm,amssymb}
\usepackage{enumerate}
\usepackage{xcolor}
\usepackage{cite}
\usepackage{url}
\usepackage{xspace}
\usepackage{hyperref}
\hypersetup{
    colorlinks=true
}

\usepackage{commands}
\usepackage[utf8]{inputenc}
\usepackage{amsfonts}
\usepackage{float}
\usepackage{graphicx}
\usepackage{bm}
\usepackage{amssymb}
\usepackage{bbm}
\usepackage{xcolor}
\usepackage{algorithm}
\usepackage{algorithmic}
\usepackage{dsfont}
\usepackage{xspace}
\usepackage{comment}
\usepackage{tikz}

\newcommand{\eat}[1]{}

\newcommand{\addtag}{\addtocounter{equation}{1}\tag{\theequation}}

\newtheorem{result}{Theorem} \newtheorem{minorresult}[result]{Proposition}

\newtheorem{theorem}{Theorem}[section]
\newtheorem{lemma}[theorem]{Lemma}
\newtheorem{definition}[theorem]{Definition}

\newtheorem{proposition}[theorem]{Proposition}
\newtheorem{corollary}[theorem]{Corollary}

\newtheorem{example}[theorem]{Example}

\newtheorem*{proposition*}{Proposition}

\newtheorem*{definition*}{Definition}

\newcommand{\X}{\mathcal{X}}
\newcommand{\Y}{\mathcal{Y}}
\newcommand{\mC}{\mathcal{C}}
\newcommand{\mD}{\mathcal{D}}

\newcommand{\mB}{\mathcal{B}}

\newcommand{\mI}{\mathcal{I}}
\newcommand{\mL}{\mathcal{L}}
\newcommand{\mW}{\mathcal{W}}

\newcommand{\mP}{\mathcal{P}}

\newcommand{\mcal}{\mathrm{MC}}
\newcommand{\mcals}{\mathrm{MC}^\mathrm{s}}
\newcommand{\mcalreg}{\mathrm{MC}^\mathrm{f}}
\newcommand{\macc}{\mathrm{MA}}

\newcommand{\y}{\mathbf{y}}
\newcommand{\x}{\mathbf{x}}
\newcommand{\z}{\mathbf{z}}

\newcommand{\rgta}{\rightarrow}

\newcommand{\lt}{\left}
\newcommand{\rt}{\right}

\newcommand{\zo}{\ensuremath{\{0,1\}}}

\newcommand{\norm}[1]{\left\lVert#1\right\rVert}

\newcommand{\infnorm}[1]{\left\lVert#1\right\rVert_\infty}

\renewcommand{\varepsilon}{\epsilon}
\renewcommand{\tilde}{\wt}

\renewcommand{\eps}{\epsilon}

\newcommand{\fr}[1]{\ensuremath{\frac{1}{#1}}}

\newcommand{\C}{\mathcal{C}}

\DeclareMathOperator{\poly}{poly}

\DeclareMathOperator*{\E}{\mathbf{E}}

\newcommand{\card}[1] {\left\vert #1 \right\vert}
\newcommand{\set}[1] {\left\{ #1 \right\}}
\newcommand{\given}{~\middle \vert~}
\newcommand{\lr}[1]{\left[~ #1 ~\right]}

\newcommand{\fs}{f^*}

\newcommand{\Lip}{\mL_{1\to\infty}}

\newcommand{\VC}{\mathsf{VC}}

 \usepackage{enumerate}

\title{Low-Degree Multicalibration}
\author{
Parikshit Gopalan\thanks{\texttt{pgopalan@vmaware.com}}\\
VMware Research
\and
Michael P. Kim\thanks{\texttt{mpkim@berkeley.edu}~~Supported by the Miller Institute for Basic Research in Science and by the Simons Collaboration on the Theory of Algorithmic Fairness}\\
UC Berkeley
\and
Mihir Singhal\thanks{\texttt{mihirs@mit.edu}~~This work completed during an internship at VMware Research.}\\
MIT
\and
Shengjia Zhao\thanks{\texttt{sjzhao@stanford.edu}}\\
Stanford University
}

\begin{document}
\maketitle

\begin{abstract}
Introduced as a notion of algorithmic fairness, multicalibration has proved to be a powerful and versatile concept with implications far beyond its original intent.
This stringent notion---that predictions be well-calibrated across a rich class of intersecting subpopulations---provides its strong guarantees at a cost: the computational and sample complexity of learning multicalibrated predictors are high, and grow exponentially with the number of class labels.
In contrast, the relaxed notion of multiaccuracy can be achieved more efficiently, yet many of the most desirable properties of multicalibration cannot be guaranteed assuming multiaccuracy alone.
This tension raises a key question:  \emph{Can we learn predictors with multicalibration-style guarantees at a cost commensurate with multiaccuracy?}

In this work, we define and initiate the study of \emph{Low-Degree Multicalibration}.
Low-Degree Multicalibration defines a hierarchy of increasingly-powerful multi-group fairness notions that spans multiaccuracy and the original formulation of multicalibration at the extremes.
Our main technical contribution demonstrates that key properties of multicalibration, related to fairness and accuracy, actually manifest as low-degree properties.
Importantly, we show that low-degree multicalibration can be significantly more efficient than full multicalibration.
In the multi-class setting, the sample complexity to achieve low-degree multicalibration improves exponentially (in the number of classes) over full multicalibration.
Our work presents compelling evidence that low-degree multicalibration represents a sweet spot, pairing computational and sample efficiency with strong fairness and accuracy guarantees.
\end{abstract}

\clearpage
\section{Introduction}

Machine learning models are increasingly used to aid decision-making in professional, personal, medical, and legal spheres.
This ubiquity has brought increased concern about whether these models make fair predictions, especially on underrepresented \emph{subpopulations}.
Typically in supervised learning, models are trained to minimize the expected loss over the \emph{entire population}, and can be less accurate on such subpopulations.
A particular fairness concern is that predictive models may commit \emph{algorithmic stereotyping}, where every member of a subpopulation receives similar predictions, despite internal diversity within the subpopulation.
Such shortcomings of the standard supervised learning framework have been documented extensively within the research community and in the popular press.
In response, a growing area of research investigates \emph{multi-group} fairness notions, that require predictions to perform well not simply overall, but even when restricting attention to structured subgroups \cite{hkrr2018,kearns2018,krr,kgz,shabat2020sample,blum2020advancing,OI,Jung20,RothblumYona21,tosh2021simple,gupta2021online,dwork2022beyond}.

Central within the study of multi-group fairness is the notion of \emph{multicalibration}.
Calibration is a classic notion from the forecasting literature \cite{Dawid} that was introduced to the literature on fairness in prediction tasks by \cite{KleinbergMR17}.
For a binary prediction task, calibration requires that amongst the individuals which receive prediction $f(x) = v$, the true expectation is $v$.
Defined by \cite{hkrr2018}, multicalibration strengthens the classic notion, requiring a predictor to be calibrated simultaneously across a large, possibly-overlapping collection of subpopulations.
We model the collection by a hypothesis class $\mC \subseteq \set{c:\X \to \set{0,1}}$ and say that a predictor $f$ is mutlicalibrated over $\mC$ if, for all predicted values $v \in [0,1]$, and for all $c \in \mC$
\begin{gather*}
\E\lr{c(\x) \cdot (\y - v) \given f(\x) = v} \approx 0.
\end{gather*}
Intuitively, an expressive class $\mC$ will contain subpopulations that go beyond ``protected groups'' (typically defined marginally in terms of a single attribute).
In this way, while calibration provides only marginal guarantees, multicalibration requires predictors to capture the variation within subpopulations, to give confident (but not overconfident) predictions,
thus providing strong protections against algorithmic stereotyping.

Beyond its origins as a notion of fairness, multicalibration has proved surprisingly versatile and powerful in diverse contexts.
The notion of multicalibration makes no mention of loss minimization, yet the work of \cite{omni} shows that multicalibrated predictors implicitly obtain optimal loss, simultaneously for all Lipschitz, convex losses.
Specifically, given any (fixed) $\mC$-multicalibrated predictor $f$, for every such loss $\ell$, the predictor $f$ guarantees loss competitive with the hypothesis $c_\ell \in \mC$ chosen to the minimize the loss over $\mC$ (in fact, over linear combinations of hypotheses in $\mC$).
This property leads to an {\em omniprediction} guarantee, where one can learn a single predictor $f$, without knowledge of the choice of loss at the time of learning.

In another direction, \cite{OI} demonstrate that the multicalibration framework is equivalent to a certain pseudorandomness condition, which they call \emph{outcome indistinguishability}.
Intuitively, a predictor $f$ is outcome indistinguishable to a family of distinguisher algorithms $\mathcal{A}$ if given a sample $(\x,\y)$ no distinguisher $A \in \mathcal{A}$ can tell whether $\y$ was sampled from Nature's true conditional distribution of $\y\vert \x$, or according to the predicted probability $f(\x)$.
This indistinguishability perspective has seen application in characterizing the feasibility of multi-group strengthenings of agnostic PAC learnability \cite{RothblumYona21}.
Multicalibration has also been extended to diverse settings of uncertainty quantification for real-valued outcomes \cite{Jung20}, importance weights \cite{gopalan2021multicalibrated}, online prediction \cite{gupta2021online}, and adaptation under covariate shift \cite{kim2022universal}.

It is perhaps not a surprise that the power of  multicalibration comes at a cost---both in terms of samples and computation.
Computationally, \cite{hkrr2018} show that a \emph{weak agnostic learner} for the class $\mC$ is necessary and sufficient to learn $\mC$-multicalibrated predictors.
Using the weak learner, they design a boosting-style algorithm that produces a multicalibrated predictor by combining hypotheses $c \in \mC$ using nontrivial Boolean logic.\footnote{In fact, multicalibration has been shown to be tightly connected to the boosting-by-branching-programs framework of \cite{MansourM2002,KalaiMV08}. See \cite{omni} for a discussion.}
The number of calls to the weak learner and the sample complexity are governed by an approximation parameter $\alpha$ which quantifies the {\em deviation} from perfect multicalibration.
The sample complexity depends inverse polynomially in the parameter $\alpha$, with fairly large exponent.
This dependence becomes even worse when we generalize multicalibration to the multi-class setting with $l > 2$ class labels, where a single prediction is a vector of probabilities in $l$ dimensions. Reasoning about expectations conditioned on the predictions leads to complexity that scales as $1/\alpha^{\Omega(\ell)}$, exponential in the number of class labels.
Consequently, achieving multicalibration is practically infeasible with more than a few classes.

In the work defining multicalibration, \cite{hkrr2018} also introduced a weaker fairness notion known as \emph{multiaccuracy}, which only requires that $f(\x)$ and $\y$ have similar expectations over subpopulations in $\mC$, without conditioning on the predicted values.
\begin{gather*}
    \E\lr{c(\x) \cdot (\y - f(\x))} \approx 0
\end{gather*}
The simpler notion of multiaccuracy is comparatively easier to obtain, quantitatively and qualitatively. One can view $\mC$-multiaccuracy as a first-order optimality condition on $\lambda \in \R^{\card{\mC}}$ for predictors of the form $f(x) = \sum_{c \in \mC} \lambda_c \cdot c(x)$.
Thus, while multiaccuracy also requires a weak learning oracle, one can learn a multiaccurate predictor simply by  minimizing the squared or logistic loss over linear combinations of $c \in \mC$ (without specialized boolean logic), using techniques like coordinate ascent or gradient boosting.
Further, standard concentration inequalities demonstrate that the sample complexity to obtain multiaccuracy scales as $\alpha^{-2}$, even as the number of classes $l$ grows.

In exchange for its efficiency, multiaccuracy is known to provide considerably weaker guarantees than multicalibration.
Many of the most desirable fairness and accuracy properties that can be derived from multicalibration cannot be derived from multiaccuracy alone, including its guarantees for loss minimization \cite{hkrr2018,RothblumYona21,omni}, the fairness properties of the ranking induced by predictions \cite{dwork2019rankings}, and multi-group confidence intervals \cite{Jung20}.
This tension---between guarantees and efficiency---brings us to the motivating question behind our work:
\begin{center}
    \emph{Are there notions that retain important properties of multicalibration,
    \\but are computable much more efficiently (comparable to multiaccuracy)?}
\end{center}

\section{Overview of Contributions}
\label{sec:overview}

We introduce a hierarchy of  multicalibration notions that enable a tradeoff between the strength of multi-group guarantees and the complexity required to learn the predictor.
\eat{
In this work, we make progress on answering the motivating question by developing} 
At the extremes, our hierarchy recovers the existing notions of multiaccuracy and multicalibration, but our interest is in the  intermediate notions.
We establish guarantees about these notions, and show that many desirable properties of multicalibration kick in at low levels of the hierarchy. In doing so, we gain new insights into the power of (full) multicalibration. We complement these with algorithmic results showing that computing predictors in the low levels of the hierarchy can be significantly more efficient than the original formulation of multicalibration.
Our three main contributions can be summarized as follows:
\begin{enumerate}[(1)]
\item
Our primary conceptual contribution is the definition of \emph{low-degree multicalibration} and its associated hierarchy.
For $k \in \N$, the $k$th level of the hierarchy defines a notion of multicalibration that constrains the first $k$ moments of the predictor, conditioned on subpopulations in $\mC$. The lowest level of the hierarchy corresponds to multiaccuracy.
As we go higher, the multicalibration constraints become more stringent. In the limit, we approach a notion we call \emph{smooth} multicalibration; a relaxation of the original formulation of \cite{hkrr2018}, which we refer to as \emph{full} multicalibration.
\eat{We demonstrate that as we move up the low-degree hierarchy, the limiting notion recovers the strongest notions of smooth and full multicalibration.}
\item With the hierarchy in place, we study the fairness and accuracy properties obtainable via low-degree multicalibration.
Our main contribution is to provide a rich toolbox for reasoning about multicalibrated predictions $f(x)$, by comparing to the moments of the Bayes optimal predictions $\fs(x) = \E\lr{\y \given \x = x}$.
Our key technical result establishes novel {\em moment sandwiching bounds} for multicalibrated predictors.
For instance, in lieu of $k^{th}$ moment matching (which would give $E[f(\x)^k] \approx \E[f^*(\x)^k]$, but is impossible to achieve for $k >1$), we show that degree-$k$ multicalibration implies that $\E[f(\x)^k] \leq \E[f^*(\x)^k]$, even when conditioned on subpopulations defined by $c \in \mC$. 
Using these tools, we can relate the confusion rates (generalized false error rates) of any degree-$2$ multicalibrated $f$ to those of the optimal predictor.
Our results reveal wide gaps even between multiaccuracy and degree-$2$ multicalibration; predictors satisfying the latter cannot exhibit overconfidence over subpopulations in $\mC$, unlike multiaccurate predictors. 

\item Finally, we show that low-degree multicalibration can provide significant savings over full multicalibration.
In particular, we show that for $l$-class prediction tasks, the sample complexity to obtain low-degree multicalibration is polynomial in $l$,  whereas obtaining the same guarantees using full multicalibration requires sample complexity exponential in $l$. Even in the case of binary prediction, low-degree multicalibration obtains improvements over full multicalibration, by polynomial factors in the approximation parameter $1/\alpha$.
\eat{We show that low-degree multicalibration provides exponential savings in sample complexity for multiclass predictors.  Even for binary predictors, we obtain polynomial savings  that}
These bounds suggest that the low-degree notions may be practically-realizable, providing strong guarantees in settings where the existing notions of multicalibration cannot be achieved.
\end{enumerate}
In all, we develop a more refined picture of multiaccuracy, multicalibration, and the guarantees that lie in between.
Our results establish that---in addition to the calibration class $\mC$ and approximation parameter $\alpha$---the degree of multicalibration is a meaningful ``knob'' that can be tuned to the needs and constraints of a given setting.
Low-degree multicalibration provides a new perspective and set of techniques that we anticipate will be useful to practitioners and theoreticians alike.

\paragraph{Organization of manuscript.}
The remainder of the manuscript is structured as follows.

First, we continue this section with a high-level overview of our contributions, focusing on the binary prediction setting.
In Section~\ref{overview:motivation}, we begin with an intuitive explanation of how one might discover low-degree multicalibration.
Then, in Section~\ref{overview:hierarchy}, we present the definitions of the new variants of multicalibration and their relation to one another; in Section~\ref{overview:moments}, we present the novel moment sandwiching bounds for multicalibrated predictors; and in Section~\ref{overview:samples}, we present the bounds on the complexity of achieving each variant of multicalibration.
We conclude the overview with Section~\ref{sec:discussion}, where we provide further discussion of low-degree multicalibration and how it relates to other works on multi-group fairness, agnostic learning, and calibration.

The technical portions of the manuscript follow the structure of the overview.
In Section~\ref{sec:high-def}, we give formal definitions of the notions within the low-degree multicalibration hierarchy, handling the multi-class setting.
Then, we establish Proposition~\ref{prop:inclusions} and other the relationships and robustness properties of notions in the hierarchy.
In Section~\ref{sec:properties}, we establish our main technical result, Theorem~\ref{thm:informal:main}, along with other key fairness properties of low-degree multicalibration.
In Section~\ref{sec:alg}, we describe Algorithm~\ref{alg:wmc} and analyze the sample complexity as in Theorem~\ref{thm:samples:informal}.
In Section~\ref{sec:experiments}, we highlight a proof-of-concept experimental evaluation of low-degree multicalibration.

\subsection{Towards Multi-Group Moment Matching}
\label{overview:motivation}
In this motivating vignette, we focus on the setting of binary prediction:  we are given samples from a distribution $\mD$ on $\X \times \Y$, for domain $\X$ and label space $\Y = \zo$.
We use $f:\X \to [0,1]$ to denote our hypothesis, and use $\fs:\X \to [0,1]$ to denote the Bayes optimal predictior, defined as
$\fs(x) = \E\lr{\y \given \x = x}$,
the true expected outcome over $\mD$ of $\y$ given $\x = x$.\footnote{We use boldface for random variables.  All expectations are taken over $\mD$.}
Ideally, our learned hypothesis $f$ should be a close approximation to $f^*$.
Every variant of multicalibration is parameterized by a hypothesis class $\mC \subseteq \set{c:\X \to \set{0,1}}$.
We think of $\mC$ as an expressive but bounded class, where the functions $c \in \mC$ have a simple representation; for instance, we may assume that the VC-dimension of $\mC$ is finite.

To begin, we describe the intuition for multiaccuracy and how one might strengthen it, without appealing to full multicalibration.
Towards the goal of approximating $\fs$, multiaccuracy imposes a first-order condition of matching expectations over $c \in \mC$.
\begin{gather*}
\E\lr{c(\x) \cdot (\y - f(\x))} \approx 0 \implies \E\lr{c(\x) \cdot \fs(x)} \approx \E\lr{c(\x) \cdot f(\x)}
\end{gather*}
Naturally, the next step might be to try and match second moments, and require that $\E[c(\x) \cdot f(\x)^2] \approx \E[c(\x) \cdot f^*(\x)^2]$.
This ask, however, is information-theoretically infeasible.
In our setting, we only get to see samples $(\x, \y) \sim \mD$ for discrete $\y \sim \mathrm{Ber}(\fs(\x))$ and do not have access to the values $f^*(\x)$ themselves.
Initially, in such a setting, it seems impossible to say anything meaningful about higher-order moments of $\fs(\x)$.

Short of second-moment matching, degree-$2$ multicalibration imposes the following constraint for all $c \in \mC$.
\begin{gather*}
\E\lr{c(\x) \cdot f(\x) (\y - f(\x))} \approx 0 \implies \E\lr{c(\x) \cdot f(\x)\fs(\x)} \approx \E\lr{c(\x) \cdot f(\x)^2}
\end{gather*}
As a sanity check, observe that this is a {\em valid} constraint, since $f =f^*$ does satisfy it. As in multiaccuracy, the condition can be audited using only access to samples $(\x,\y)$ and the predictions $f(\x)$, without knowledge of $\fs(\x)$. But what do we gain from adding this constraint?
As we will see, this ``degree-$2$'' constraint turns out to be surprisingly powerful.
Short of second moment matching, we prove that the degree-$2$ condition (approximately) implies the following second moment inequality:
\begin{gather*}
\E\lr{c(\x) \cdot f(\x)^2} \le \E\lr{c(\x) \cdot \fs(\x)^2}
\end{gather*}
It is unclear that this inequality by itself can be audited from samples alone, yet it is implied by the conditions of degree-$2$ multicalibration, which are indeed auditable.
Intuitively, the inequality says that---unlike multiaccuracy---this degree-$2$ variant of multicalibration prevents overconfident predictions conditioned on $c \in \mC$. It gives us an avenue to reason about $f(x)$ by comparing it to $\fs(x)$, using its first two moments.

Concretely, the set of degree-$2$ constraints is sufficient to give nontrivial guarantees on the agnostic learning properties of multicalibrated predictors.
A simple calculation\footnote{Included in Appendix~\ref{app:lossmin}} shows that if $f$ satisfies degree-$2$ multicalibration over the class $\C$, then $f$ achieves squared error comparable to the minimum over functions $c \in \C$.
\begin{gather*}
    \E\lr{(f(\x) - \fs(\x))^2} \le \min_{c \in \C} \E\lr{(c(\x) - \fs(\x))^2}
\end{gather*}
This loss minimization guarantee is similar to the ``omnipredictor'' loss minimization guarantees from full multicalibration, recently established in \cite{omni}.
In effect, for a more restricted class of convex losses like the squared error, omniprediction is actually a low-degree property.

The multicalibration hierarchy arises by lifting this intuition to higher degree polynomials: at the $k$th level, we obtain guarantees on the $k$th moments.
Indeed, we show that a number of meaningful guarantees that hold for low-degree  multicalibration, but not for multiaccuracy.
Our results are in direct analogy with classic results in pseudorandomness, where $k$-wise---even pairwise---independence is known to be surprisingly powerful, in contrast to $1$-wise independence \cite{LW06}. As with pseudorandomness, we are able to prove new guarantees for multicalbration in its full generality, by showing that they hold even for low-degree multicalibration.

\subsection{A Hierarchy of Multicalibration}
\label{overview:hierarchy}
With the motivation in place, we are ready to define the multicalibration hierarchy.
In this overview, we present the notions focusing on binary predictors.
In Section~\ref{sec:high-def}, we extend these definitions to the multi-class setting where outcomes are categorical random variables with $l \ge 2$ labels.
As with previous notions, our variants of multicalibration are parametrized by a hypothesis class $\mC$ and an approximation parameter $\alpha > 0$.
A new ingredient of our definitions will be a family $\mW\subseteq \set{w:[0,1] \rgta [0,1]}$ of {\em weight} functions that we will compose with the predictions of our model, which gives rise to a generic weighted version of multicalibration.
\begin{definition*}
Given a hypothesis class $\mC$ and a weight class $\mW$, we say that the predictor $f: \X \rgta [0,1]$ is $(\mC,\mW, \alpha)$-multicalibrated if for every $c \in \mC$ and $w \in \mW$ it holds that
    \begin{equation} 
\card{\E_\mD\lr{c(\x) \cdot w(f(\x))(\y - f(\x))}} \le \alpha.
    \end{equation}
\end{definition*}
The primary conceptual contribution of this work is to identify choices of weight classes $\mW$ that give rise to novel meaningful notions of multicalibration.
We consider four notions, which fixing a hypothesis class $\mC$, give increasingly stronger guarantees.
\begin{enumerate}[(1)]
\item {\bf Multiaccuracy.} Taking $\mW =\{\mathrm{1}\}$ to consist of only the constant function $w(z) = 1$ for all $z$, we recover {\em multiaccuracy}.\footnote{We adopt the formulation of multiaccuracy initially defined in \cite{kgz}.} Let $\macc(\alpha)$ denote the set of $(\mC,\alpha)$-multiaccurate predictors.
\item {\bf Low-degree multicalibration.}
We define a hierarchy of variants of low-degree multicalibration, by taking $\mW$ to be families of low-degree polynomials.
Formally, \emph{degree-$k$} multicalibration uses weight functions defined by sparse degree-$(k-1)$ polynomials.
We adopt this convention because using degree-$(k-1)$ polynomials as weight functions allow us to reason about the $k$th moments of our predictions.
In the case of binary prediction, it suffices to take $\mW_k =\{t^j\}_{j =0}^{k-1}$ to be the monomial basis.
Let $\mcal_k(\alpha)$ denote the set of $(\mC,\alpha)$-degree-$k$ multicalibrated predictors.
\item {\bf Smooth multicalibration.} Beyond the hierarchy, we consider multicalibration using the family all $1$-Lipshcitz functions as our weight class.
We refer to the resulting notion
as $(\mC,\alpha)$-\emph{smooth} multicalibration.
Smooth multicalibration directly extends the notion of smooth calibration \cite{kakadeF08, FosterH18}, introduced to address issues of robustness in defining calibration.
Let $\mcals(\alpha)$ denote the set of $\alpha$-smooth multicalibrated predictors.
\item {\bf Full multicalibration.}
Instead of explicitly conditioning on the predicted values, we define full multicalibration using indicator functions on the prediction intervals as our weight class.
In binary prediction,\footnote{For multi-class prediction, we use an $l$-dimensional analog of the interval basis.} we define the  $\delta$-interval basis to be $\mI_\delta = \{\ind{[(j-1)\delta,j\delta)} : j \in \lceil 1/\delta \rceil\}$ where $\ind{[a,b)}$ indicates membership in the interval $[a,b)$.
We refer to $(\mC, \mI_\delta, \alpha)$-multicalibration as $(\mC,\alpha, \delta)$-{\em full} multicalibration to distinguish it from smooth multicalibration, and use $\mcalreg_\delta(\alpha)$ to denote the set of predictors satisfying it.
\end{enumerate}
With the variants in place, our first results establish the relationship between the multicalibration notions for a fixed class $\mC$.
\begin{minorresult}
\label{prop:inclusions}
Fix a hypothesis class $\mC$ and $\alpha \ge 0$.
By construction, multiaccuracy and degree-$1$ multicalibration are equivalent.
For every $k \ge 1$, increasing the degree leads to a more restrictive notion.
In other words, the hierarchy satisfies the following inclusions.
\begin{gather*}
    \macc(\alpha) = \mcal_1(\alpha) \supseteq \mcal_2(\alpha)  \cdots \supseteq \mcal_k(\alpha)
\end{gather*}
Further, low-degree multicalibration is a relaxation of smooth multicalibration, which is a relaxation of full multicalibration.
That is, for any $k \ge 1$ and $\alpha \ge \delta \ge 0$,
\begin{gather*}
    \mcal_k(k\alpha) \supseteq \mcals(\alpha) \supseteq \mcalreg_{\delta}(\alpha\delta - \delta^2)
\end{gather*}
\end{minorresult}

\eat{
We also investigate the properties of multicalibration, parameterized by a weight class $\mW$, providing further justification for its utility.
\begin{itemize}
    \item First, [robustness to choice of weight functions]. Leads to finite approximation by basis.
    \item S
\end{itemize}
}

\cite{kakadeF08,FosterH18} were motivated to introduce smooth calibration  in order to address the lack of robustness in the classical notion of calibration.
We show that smooth multicalibration has similar robustness guarantees; for instance, predictors that are close to smoothly multicalibrated functions are also smoothly multicalibrated (with modest degradation in the approximation parameter).
\eat{
These papers discuss why smooth calibration is a more robust notion than regular calibration.
We prove analogous robustness results in the multicalibration setting.
We show that both low-degree and smooth multicalibration are robust to small (in $\ell_1$) perturbations of the predictor.
This is in contrast to regular multicalibration, where even a small perturbation can make a multi-calibrated predictor far from calibrated.}
We also show robustness to small (in $\ell_\infty$) perturbations to the weight function. This property is important algorithmically, since it allows us to infer smooth multicalibration (defined over all $1$-Lipshcitz weight functions) from a small basis of functions that can uniformly approximate every 1-Lipschitz function in $\ell_\infty$.

\subsection{Fairness from Low-Degree Multicalibration}
\label{overview:moments}

We now return to our motivating question and ask: {\em does the hierarchy give strong fairness guarantees at low levels?}
Our main technical contribution is a toolbox for reasoning about properties of predictors at low levels of the hierarchy, using their first few moments. Using this, we establish that the gurarantees of multicalibration for important measures of fairness, like the false error rates
indeed manifest at low levels of the hierarchy.
\eat{
In particular, we establish \emph{moment sandwiching bounds} for multicalibrated predictors that are useful for reasoning about 
Specifically, we show that the $k$th moments of a $(\mC,\alpha)$-degree-$k$ multicalibrated $f$, evaluated over $c \in \mC$, are ``sandwiched'' between quantities defined in terms of $\fs$.}
The hammer in this toolbox is the following  \emph{moment sandwiching bound} stated below for binary predictors with $\alpha =0$.

\begin{result}[Informal]
\label{thm:informal:main}
Suppose that $f:\X \to [0,1]$ is a $(\mC,0)$-degree-$k$ multicalibrated predictor.
Then,
for every degree $d \le k$ and every $c \in \mC$,
\begin{gather*}
\E[c(\x) \fs(\x)^d] \underset{(a)}{\ge} \E[c(\x) f(\x)^d] \underset{(b)}{=} \E[c(\x) f(\x)^{d-1}\fs(\x)] \underset{(c)}{\ge} \E[c(\x) f(\x)^{d-1}] \cdot \E[c(\x)\fs(\x)]
\end{gather*}
\end{result}
Our bounds are inspired by the sandwiching bounds for importance weights from multicalibrated partitions proved in \cite{gopalan2021multicalibrated}.
While we prove these bounds assuming only degree-$k$ multicalibration, no analogous statements were known to hold, even for the original (stronger) notion of multicalibration.  
\eat{
We reiterate that the upper bound in (a) seems hard to audit by itself, but it is implied by the auditable constraints of low-degree multicalibration.

An interesting facet of this sandwiching bound on the $k$th moment is that the upper bound of $\E[f^*(\x)^k]$ from $(*)$ is hard to verify.
Because we only have access to labels $\y \in \{0,1\}$, rather than values of the optimal predictor $f^*(\x) \in [0,1]$, it is impossible to estimate $f^*(x)^k$ for $k > 1$.
}

The utility of this sandwiching bound can be seen by thinking about each inequality separately.
First, note that $(b)$ follows immediately by the definition of degree-$k$ multicalibration.
We dig into $(a)$ and $(c)$ separately.
The upper bound in $(a)$ says that first $k$ moments of the multicalibrated predictor $f$ are dominated by those of the ground truth predictor $f^*$, conditioned on any $c \in \mC$.
Given that the exact $k$th moment of $\fs$ is inaccessible through samples, this inequality is useful in lieu of exact moment matching.

Specifically, while we can't perfectly match moments for $d \geq 2$, this inequality implies that the predictions of $f$ cannot be overconfident.
For example, for a degree-$2$ multicalibrated $f$, the inequality implies that conditioned on any $c \in \mC$, the variance of $f$ is no more than the variance of  $f^*$, and their expectations match.
In other words, any variation in the  predictions of $f$ over $c \in \mC$ can be attributed to true variation in the distribution of $f^*(\x)|c(\x) = 1$ .
Concretely, if all points with $c(\x) = 1$ were identical under $f^*$, then $f$ cannot treat them differently!
This stands in stark constrast with multiaccuracy, where many of the ``failure modes'' of multiaccuracy exploit this weakness \cite{dwork2019rankings,OI}.

The lower bound in $(c)$ can be interpreted as saying that $f^{k-1}(\x)$ is positively correlated with the label $\y$, by expressing the difference in expectations as the covariance between $f^{k-1}(\x)$ and $\y$.
\begin{gather*}
\Cov[f(\x)^{k-1}, \y] = \E[f(\x)^{k-1}\fs(\x)] - \E[f(\x)^{k-1}]\E[\fs(\x)] \ge 0
\end{gather*}
In this light, the bounds from degree-$2$ multicalibration can be summarized succinctly:
the covariance between $f(\x)$ and the ground truth labels $\y$ within any $c \in \mC$ is always positive, but never more than the covariance between the optimal $f^*(\x)$ and $\y$ in $c$.
Again, neither claim necessarily holds under multiaccuracy.

As a concrete application, we show how sandwiching bounds allow us to reason  about the true and false error rates of $f$, which have been intensively studied in algorithmic fairness \cite{hardt2016equality,KleinbergMR17}.
In particular, the bounds extend even when we condition on the value of the true label.
For instance, conditioning on $\y = 1$, we can bound the generalized true positive rate of $f$ compared to that of $\fs$, even conditioned on subpopulations.\footnote{An analogous statement can be made about the true negative rates.}
\begin{gather*}
\E\lr{c(\x)f(\x)} \le \E\lr{c(\x)f(\x) \given \y = 1} \le \E\lr{c(\x)\fs(\x) \given \y = 1}
\end{gather*}
This bound crystallizes the inutition that multicalibrated predictors prevent overconfidence.
Even though $f$ has positive correlation with $\y$ over $c$, the predictions don't assign excessively high probabilities to points which are more likely to be $1$. 
Such overconfidence is a well-known issue in large neural networks~\cite{guo2017calibration}. Whereas full multicalibration addresses overconfidence at a  fine-grained level, conditioning on every value of the prediction, even degree-$2$ multicalibration gives a qualitatively similar guarantee.
\eat{
These bounds reiterate the idea that multicalibration post-processing may be an effective way to recalibrated predictions \cite{kgz}.
}

\subsection{The Pragmatic Appeal of Low-Degree Multicalibration}
\label{overview:samples}

We establish the feasibility of low-degree and smooth multicalibration by giving an efficient learning algorithm. More generally, following the boosting-style approach of \cite{hkrr2018}, we show that a weak agnostic learner for $\mC$ \cite{SBD2,kk09} suffices to obtain $(\mC,\mW,\alpha)$-multicalibration for any finite weight class $\mW$.
This procedure, described as Algorithm~\ref{alg:wmc}, makes a number of calls to the weak learner bounded polynomially in $\card{\mW}$ and $1/\alpha$.
The algorithm immediately demonstrates the feasibility of degree-$k$ multicalibration for any fixed $k \in \N$. For smooth multicalibration, where the definition involves an infinite family $\mW$, our algorithm uses constructions of finite bases that uniformly approximate the family of $1$-Lipschitz functions.

We now delve further into the sample efficiency of each notion.
In principle, the relaxed notion of low-degree multicalibration is implied by smooth or full multicalibration. However, there is a significant degradation in the accuracy parameter $\alpha$, which translates to a blowup in computational and sample complexity. 
\eat{

This argument, however, breaks down when we consider how the concrete approximation guarantee degrades when deriving the low-degree condition from the stronger notions of calibration.}
Concretely, suppose we train a predictor $f$ to satisfy $(\C,\alpha_0,\delta)$-full multicalibration; in order to guarantee that $f$ satisfies $(\mC,\alpha)$-degree-$k$ multicalibration, we need to take $\alpha_0 \ll \alpha$ much smaller than if we train for degree-$k$ multicalibration directly. For a fair comparison, we compare the sample complexity needed by each variant to obtain the same guarantee: $ (\mC, \alpha)$ degree-$k$ multicalibration.
We make the assumption that $\mC$ has a sample-optimal weak agnostic learner in terms of the dependence on $\alpha_0$.\footnote{Suboptimal dependence of the learner on $\alpha_0$ will \emph{increase} the gaps in sample complexity between low-degree and the stronger variants of multicalibration.}
We state the theorem informally, using shorthand $m \sim B$ to denote $m \le \tilde{O}(B)$ (see Section~\ref{sec:alg}, Theorem~\ref{thm:samples} for a formal statement).

\begin{result}[Informal]
\label{thm:samples:informal}
For a $\mC$ be a hypothesis class, let $\VC(\mC)$ denote its VC-dimension.
For any $\alpha > 0$ and $k \in \N$, the sample complexity to obtain a predictor $f \in \MC_k(\alpha)$ (with constant failure probability) is bounded as $m_k$ using degree-$k$ multicalibration, $m_s$ using smooth multicalibration, and $m_i$ using full multicalibration, for $m_k,m_s,m_i$ bounded as follows.
    \begin{align*}
m_k \sim \frac{l \cdot \left(\VC(\mC) + k \right)}{\alpha^4}&&
m_s \sim \frac{k^4 l \cdot \VC(\mC)}{\alpha^4} + \frac{k^{l+3}l^l}{\alpha^{l+3}}&&
m_i \sim \frac{(2kl)^{4(l+1)} \cdot \VC(\mC)}{\alpha^{4(l+1)}}
    \end{align*}
\end{result}
Note the dependence is $\poly(l, 1/\alpha)$ for low-degree multicalibration, but $\Omega(1/\alpha)^{b l}$  for smooth and full multicalibration, with a factor $4$ difference in the constant $b$ between them. This suggests that in the multi-class setting, low-degree multicalibration can lead to exponential savings, as compared to the stronger notions. Even under tighter analyses tailored to the binary prediction setting, we obtain polynomial savings in $1/\alpha$ from  low-degree multicalibration.

Given that many desirable properties of multicalibration start to take effect at small degrees (even degree-$2$), these results suggest a pragmatic win for low-degree multicalibration. In the realistic setting where a learner is given a fixed data set of sample, the learner may achieve stronger fairness and accuracy properties by training for low-degree multicalibration directly, as compared to either multiaccuracy or full multicalibration.

While these sample complexities give asymptotic upper bounds, in Section~\ref{sec:experiments}, we report on a proof-of-concept experiment that explore the findings in a semi-synthetic setup.
By using a semi-synthetic setup, we can access the ground-truth $\fs$ values, thereby evaluating quantities like the gap in moments between $\fs$ and the learned predictors.
We show that, even in a standard binary prediction setting, the sample efficiency gap is not a merely asymptotic phenomenon, but is realized in a setting with a few thousand samples from the data distribution.

\subsection{Discussion and Related Works}
\label{sec:discussion}
Low-degree multicalibration adds to a growing list of works that study multicalibration and related notions.
On a conceptual level, our work is closely related to the idea of outcome indistinguishability (OI), introduced by \cite{OI,dwork2022beyond}.
OI is not a single notion, but rather an extensible framework for reasoning about the guarantees of predictions in the language of computational indistinguishability.
In particular, \cite{OI} also define a hierarchy of notions, based on the way distinguisher algorithms may access the predictions $f(\x)$.
The first two levels of their hierarchy correspond tightly to multiaccuracy and multicalibration; intuitively, multiaccuracy distinguishers do not get to observe $f(\x)$, where as multicalibration distinguishers may depend on $f(\x)$ arbitrarily.
Low-degree multicalibration refines the idea of access to $f(\x)$, where the corresponding distinguishers can functionally depend on $f(\x)$ in restricted ways.

On a technical level,
our sandwiching bounds, which shed light on how multicalibrated predictors control for uncertainty relative to the uncertainty of the optimal predictions, are actually inspired by work on unsupervised distribution learning of \cite{gopalan2021multicalibrated}.
Their work establishes analogous moment sandwiching bounds for density estimates given by so-called \emph{multicalibrated partitions}.
Closely related---but not to be confused with low-degree multicalibration---is the idea of \emph{moment multicalibration} due to \cite{Jung20}.
Moment multicalibration obtains guarantees of uncertainty quantification in the setting of \emph{real-valued} outcomes, by (full) multicalibrating on higher moments of the outcomes.
Importantly, in contrast to our setting, the relevant moments of real-valued outcomes can be estimated to arbitrary precision with enough samples.
Using these estimates, \cite{Jung20} derive Chebyshev-style confidence intervals across subpopulations for their multicalibrated predictions.

Very recently, multicalibration and OI have played a key technical role in obtaining strong guarantees for omniprediction \cite{omni} and multi-group agnostic learning \cite{RothblumYona21}.
Both results are known to follow from multicalibration, but not from multiaccuracy.
In light of our work, it would be interesting to revisit these results to see whether the proofs use the full power of full multicalibration or whether we can recover the guarantees using low-degree techniques.
More broadly, we speculate that the low-degree multicalibration hierarchy suggests a certain \emph{proof system}---akin to the convex programming sum-of-squares hierarchy \cite{barak2014sum}---in which properties derived from low-degree moments of $\fs$ could be derived for low-degree multicalibrated predictors for free.
Exploring this intuition further and how it may connect with OI is a fascinating direction for future research.

Multicalibration was initially developed as a strengthening of calibration to provide meaningful fairness guarantees, not just on the basis of marginally-defined groups, but intersectionally \cite{hkrr2018,kgz}.
The risk of inequity due to miscalibrated predictions has been well-documented \cite{kleinberg2016inherent,pleiss2017fairness,GargKR19}, especially in the setting of medical risk prediction \cite{obermeyer2019dissecting,BardaYRGLBBD21}.
Despite its origins as a complexity-theoretic fairness notion, mutlicalibration has already seen clinical application to address such equity issues \cite{BardaYRGLBBD21} and to develop a COVID-19 risk predictor in the early days of the pandemic \cite{BardaEtal20}.
While individual-level calibration is generally impossible \cite{barber2019limits}, recent works have investigated settings where guarantees at an individual-level are possible, through hedging~\cite{zhao2021right} and randomization~\cite{zhao2020individual}.

Calibration has a rich history in the forecasting literature, as a criterion for uncertainty quantification~\cite{brier1950verification,dawid1984present,FosterV98,kakadeF08,FosterH18}.
For multi-class prediction tasks, the strongest definition is known as distribution calibration~\cite{kull2015novel,song2019distribution}, and is known to require exponentially many samples in the number of labels.
Consequently, practitioners commonly use very weak notions such as confidence calibration~\cite{platt1999probabilistic,guo2017calibration},  class-wise calibration~\cite{kull2019beyond}.
To address this tension, \cite{zhao2021calibrating} recently proposed a notion of \emph{decision calibration}, which ensures that the predictions are calibrated with respect to downstream decisions, avoiding the infeasibility of calibrating to the predictions, but strengthening the marginal guarantees of confidence and class-wise calibration.
Naturally, one could extend the definition of decision calibration to decision multicalibration, with similar motivations to low-degree multicalibration but incomparable guarantees.

\section{Defining the Multicalibration Hierarchy}
\label{sec:high-def}

\paragraph{Notation.} In a generic supervised learning problem, we are given  a distribution $\mD$ on $\X \times \Y$, where $\X$ is the domain and $\Y$ is the set of labels. We consider the multi-class setting where $\Y = [l]$ for $l \geq 2$. We represent each outcome $i \in [l]$ by the ``one-hot'' encoding $e_i \in \R^l$.
We denote sampling from $\mD$ by $(\x , \y) \sim \mD$ where $\x \in \X$ and $\y \in \R^l$, so that $\y_\ell$ is the indicator for label equalling $\ell \in [l]$. Let $\Delta_l$ denote the space of probability distributions over $[l]$. 
We associate every distribution $p \in \Delta_l$ with a vector in $p \in \R^l$ where $p_\ell = \Pr_{\y \sim p}[\y = e_\ell]$.
An $l$-class predictor is a function $f: \X \rgta \Delta_l$ which maps each point to a distribution over labels.
Throughout, we use $\fs:\X \to \Delta_l$ to denote the Bayes optimal predictor, defined as
\begin{gather*}
    \fs(x) = \E\lr{\y \given \x = x}
\end{gather*}
where for each $\ell \in [l]$, $\fs_\ell(x) = \Pr\lr{\y_\ell = 1 \given \x = x}$.
In other words, $\fs(x)$ governs the true distribution over classes for a given individual $x \in \X$.
Some of our results will be stated for the special case of binary prediction $(l = 2)$, in which case, we let $f:\X \to [0,1]$, where $f(\x)$ estimates of  $\Pr\lr{\y = 1 \given \x}$.

\subsection{Multicalibration with Weight Classes}

We present a unified framework for defining variants of multicalibration that captures the original notions, as defined by \cite{hkrr2018}, but also naturally captures the novel notions of low-degree and smooth multicalibration.
We present all of the definitions in terms of $l$-class predictors, which generalizes their original definitions for binary predictors.

Multicalibration is parameterized by a hypothesis class of functions $\mC \subseteq \set{c:\X \to [0,1]}$.
The class $\mC$ may be finite or infinite, but importantly, we think of $\mC$ as having a simple representation.
Natural choice of $\mC$ include linear/logistic hypotheses, decision forests of a fixed depth, or neural networks of a fixed size and architecture.
Departing from prior works, we additionally parameterize multicalibration by a \emph{weight class} $\mW \subseteq \set{w:\Delta_l \to [0,1]^l}$.
The weight functions will be applied on the predictions from $f$.

The classic calibration constraint requires that predictions be accurate in expectation, even after conditioning on the predicted value.
Intuitively, the class of weight functions serves as an analog of ``conditioning'' on a prediction.
Taking different choices of $\mW$ will realize the original and novel variants of multicalibration.
Generically, we define multicalibration in terms of the calibration class $\mC$, the weight class $\mW$, and an approximation parameter $\alpha \ge 0$.
\begin{definition}
\label{def:mcab-multi}
Given a hypothesis class $\mC \subseteq \set{\X \rgta [0,1]}$ a weight class $\mW \subseteq \set{\Delta_l \rgta [0,1]^l}$ and $\alpha \geq 0$, a predictor $f: \X \rgta \Delta_l$ is $(\mC,\mW,\alpha)$-multicalibrated if for every $c \in \mC$ and $w \in \mW$ 
\begin{equation} 
    \label{eq:mcmc}    
\card{\E_\mD\lr{c(\x) \cdot \big\langle w(f(\x)) , \y - f(\x) \big\rangle}} \leq \alpha.
    \end{equation}
\end{definition}
With this general framework in place, we instantiate it with various choices of weight functions to derive four notions of increasing strength.

\paragraph{Multiaccuracy.}\cite{hkrr2018,kgz}
The weakest notion, multiaccuracy, requires that predictions be accurate in expectation on each $c \in \mC$.
Specifically, an $l$-class predictor $f:\X \to \Delta_l$ is $(\mC,\alpha)$-multiaccurate if for each $\ell \in [l]$,
\begin{gather}
\label{eq:mcab}    
    \card{\E_\mD\lr{c(\x) \cdot (\y_\ell - f(x)_\ell)}} \le \alpha.
\end{gather}
To instantiate $(\mC,\alpha)$-multiaccuracy from $(\mC,\mW,\alpha)$-multicalibration, we can take $\mW = \{w_\ell\}_{\ell \in [l]}$ where $w_\ell(p) = e_\ell$ for all $p \in \Delta_l$. In words, $w_\ell$ is constantly the $\ell^{th}$ standard basis vector.
In the case where $f:\X \to [0,1]$ is a binary predictor, we can simply take $\mW = \set{w_0}$ to contain the constant function $w_0(p) = 1$ for all $p \in [0,1]$.
We use $\macc(\alpha)$ to denote the set of all predictors that satisfy Equation \eqref{eq:mcab}.
\begin{gather*}
    \macc(\alpha) = \set{f:\X \to \Delta_l  \textrm{ satisfying } (\ref{eq:mcab})}
\end{gather*}

\paragraph{Low-degree multicalibration.}
Intuitively, multiaccuracy enforces a collection of linear (in $f(\x)$) constraints based on $\mC$.
In low-degree multicalibration, we strengthen these constraints by taking the class of weight functions to be the family of low-degree polynomials.
For instance, in the binary prediction setting, degree-$2$ multicalibration enforces the following nonlinear constraint.
\begin{gather*}
\card{\E_\mD\lr{c(\x) \cdot  f(\x)\cdot(\y - f(\x))}} \le \alpha.
\end{gather*}
A degree-$k$ polynomial is a function
\begin{gather*}
    q(z) = \sum_{S \in [l]^j:~j \le k} q_S \cdot \prod_{i \in S} z_i
\end{gather*}
where $S$ goes over all multisets of elements from $[l]$, of size at most $k$. $\mP_k$  is the family of all degree-$k$ polynomials that satisfy the following two conditions
 \begin{align}
 q(z) \in [0,1]~~ \forall z \in \Delta_l,\label{eqn:bounded}\\
\sum_{S \in [l]^j:~j \le k} \card{q_S} \le 1. \label{eqn:sparsity}
\end{align}
We refer to (\ref{eqn:bounded}) as boundedness and (\ref{eqn:sparsity}) as sparsity.
For $k \ge 1$, we define the weight class $\mW_k$ as all functions where each coordinate belongs to $\mP_{k-1}$.
\begin{gather*}
    \mW_k = \set{w:\Delta_l \to [0,1]^l \textrm{ such that } w_i \in \mP_{k-1}~~ \forall i \in [l]}
\end{gather*}
Note that we define the degree-$k$ weight class in terms of degree-$(k-1)$ polynomials.
We adopt this convention because, as we will see, using degree-$(k-1)$ polynomials allows us to reason about the $k$th moments of the predictor.
With this class in place, we can define degree-$k$ multicalibration.
\begin{definition}
For $k \ge 1$, a predictor $f:\X \to \Delta_l$ is $(\mC,\alpha)$-degree-$k$ multicalibrated if it is
$(\mC,\mW_k,\alpha)$-multicalibrated.
\end{definition}
Let $\mcal_k(\alpha)$ be the set of degree-$k$-multicalibrated predictors with approximation parameter $\alpha$.
\begin{proposition}
\label{prop:ma:mc1}
For any calibration class $\mC$ and approximation parameter $\alpha$, multiaccuracy and  degree-$1$-multicalibration are identical, hence
$$\macc = \mcal_1(\alpha),$$
and for every $k \ge 2$,  degree-$k$-multicalibration implies  degree-$(k-1)$-multicalibration, hence
$$\mcal_{k-1}(\alpha) \supseteq \mcal_k(\alpha).$$
\end{proposition}
In other words, low-degree multicalibration is a strengthening of multiaccuacy, and becomes a more restrictive notion as we increase the degree $k$.
The proposition follows immediately by the construction of $\mP_k$ and the fact that $\mP_{k-1} \subseteq \mP_{k}$.

\paragraph{Smooth multicalibration.}
As we increase $k$, intuitively, low-degree multicalibration will begin to approximate multicalibration with arbitrary, smooth weight functions.
This motivates our definition of smooth multicalibration using Lipschitz functions. We consider weight functions $w:\Delta_l \to [0,1]^l$ that are $\ell_1 \to \ell_\infty$ Lipschitz; that is, for all $z,z' \in \Delta_l$
\begin{gather}
\label{def:lipschitz}
\norm{w(z) - w(z')}_\infty \le \norm{z-z'}_1.
\end{gather}
We consider the class $\Lip$ of all such Lipschitz functions.
\begin{gather*}
\Lip = \set{w:\Delta_l \to [0,1]^l \textrm{ satisfying } (\ref{def:lipschitz})}
\end{gather*}
\begin{definition}
\label{def:smooth-mc}
A predictor $f:\X \to \Delta_l$ is $(\mC,\alpha)$-smoothly multicalibrated if it is $(\mC, \Lip, \alpha)$-multicalibrated.
\end{definition}
While we define smooth multicalibration in terms of $1$-Lipschitz weight functions, the property naturally extends to any weight function with bounded Lipschitz constant.
Denote $r \Lip$ as the set of weight functions
$w:\Delta_l \to [0,1]^l$, where for all $z,z' \in \Delta_l$, $\norm{w(z) - w(z')}_\infty \le r \cdot \norm{z -z'}_1$.
\begin{proposition}
\label{prop:clip}
For any calibration class $\mC$, approximation parameter $\alpha$, and constant $r > 1$,
if $f:\X \to \Delta_l$ is $(\mC,\alpha)$-smoothly multicalibrated, then for any $w \in r \Lip$,
\begin{gather*}
\card{\E_\mD\lr{c(\x) \cdot w(f(\x)) \cdot (\y - f(\x))}} \le r \cdot \alpha.
\end{gather*}
\end{proposition}
This proposition is an immediate consequence of the smooth multicalibration guarantee and linearity of expectation.

\newcommand{\ceil}[1]{\left\lceil #1 \right\rceil}

\paragraph{Full Multicalibration.} \cite{hkrr2018}
The strongest variant of multicalibration corresponds to the classic notion of calibration, where the expectation is taken conditional on the predicted value.
We present our generalization to the multi-class setting.

For a measurable set $S \subseteq \Delta_l$, let $\ind{S}$ be the indicator function of the set; that is, for any $z \in \Delta_l$,
\begin{gather*}
    \ind{S}(z) = \begin{cases} 1 & z \in S\\ 0 & z \not \in S \end{cases}
\end{gather*}
Let $\Pi$ be a partition of $\Delta_l$
and denote by $\mI_\Pi = \set{\ind{P} : P \in \Pi}$.
For binary predictors, we focus on the \emph{interval basis}.
For $\delta > 0$, we define a partition $\Pi_\delta$ of the interval $[0,1]$ to be
\begin{gather*}
    \Pi_\delta = \big\{\left[(j-1)\delta, j\delta\right) \textrm{ for } j \in \set{1,\hdots,\ceil{1/\delta}}\big\}.
\end{gather*}
We use $\mI_\delta = \mI_{\Pi_\delta}$ to denote the basis of indicator functions on each $\delta$-interval.
Taking $\mI_\delta$ as our weight class in the binary setting, $(\mC,\mI_\delta,\alpha)$-multicalibration gives us the notion we call $(\mC,\alpha,\delta)$-full multicalibration, which
recovers the original notion of multicalibration, as defined by \cite{hkrr2018}.\footnote{Technically, \cite{hkrr2018} work with a version of this notion that requires $\alpha$ error, defined \emph{relative} to the probability mass of $c^{-1}(1)$ and $\Pr[f(\x) \in [(j-1)\delta,j\delta)]$, whereas our notion gives an absolute approximation guarantee.
By adjusting the choice of $\alpha$ appropriately, it is easy to implement one notion as the other.
We follow other works adopting this absolute-error convention, to avoid many of the hassles of dealing with conditional expectations.} For $l > 2$, we define the partition $\Pi^l_\delta$ of the interval $[0,1]^l$ to be the $l$-wise Cartesian product of $\Pi_\delta$, whose elements are products of $\delta$-width intervals in each dimension. We set $\mI^l_\delta = \mI_{\Pi^l_\delta}$ and use these as weight functions in full multicalibration. We will use the easy bound $|\Pi^\delta_l| \leq \ceil{1/\delta}^l$, it also holds that $|\Pi^\delta_l| \leq l\ceil{1/\delta}^{l-1}$.

\begin{definition}
A predictor $f:\X \to \Delta_l$ is $(\mC,\alpha, \delta)$-full multicalibrated if it is $(\mC, \mI_\delta^l, \alpha)$-multicalibrated.
\end{definition}

\subsection{Understanding the Hierarchy}
\label{sec:high-def:2}

Proposition~\ref{prop:ma:mc1} demonstrates that there is a hierarchy of notions of low-degree multicalibration.
Next, we show how low-degree multicalibration compares to smooth and full multicalibration.
We show how for any $k \in \N$, for appropriately chosen approximation parameters, smooth multicalibration can implement degree-$k$ multicalibration.
Then, we show how full multicalibration can implement smooth multicalibration.

\begin{theorem}
\label{thm:deg-to-smooth}
Fix a calibration class $\mC$ and an approximation parameter $\alpha \geq 0$. For any $k \geq 2$,  every $(\mC, \alpha)$-smoothly multicalibrated predictor is $(\mC, (k-1)\alpha)$-degree-$k$ multicalibrated.
\begin{gather*}
    \mcals(\alpha) \subseteq \mcal_k((k-1)\alpha).
\end{gather*}
\end{theorem}

The theorem follows by bounding the $\ell_1 \to \ell_\infty$ Lipschitz constant of weight functions $w \in \mW_k$ by $(k-1)$ and then appealing to Proposition~\ref{prop:clip}.

\begin{lemma}
For $k \ge 2$, the degree-$k$ weight functions are $(k-1)$-Lipschitz; that is, $$\mW_k \subseteq (k-1)\mL_{1 \rgta \infty}.$$
\end{lemma}
\begin{proof}
To prove the claim, we need to show for any $z,z' \in \Delta_l$,
\begin{gather*}
\norm{w(z) - w(z')}_\infty \le (k-1) \cdot \norm{z-z'}_1.
\end{gather*}
Recall that each coordinate of $w:\Delta_l \to [0,1]^l$ can be expressed as an sparse degree-$(k-1)$ polynomial $w_i \in \mP_{k-1}$.
To obtain this $\ell_1 \to \ell_\infty$ bound, we start with the max over class predictions $i \in [l]$, and bound this quantity in terms of the gradient of $w_i \in \mP_{k-1}$.
\begin{align*}
\norm{w(z) - w(z')}_\infty &\le \max_{i \in [l]} \card{w(z)_i - w(z')_i}\\
&\le \max_{i \in [l]} \max_{z^* \in \Delta_l}\langle \nabla w_i(z^*), z - z' \rangle\addtag \label{eqn:hierarchy:mvt}\\
&\le \max_{i \in [l]}\max_{z^* \in \Delta_l}\norm{\nabla w_i(z^*)}_\infty \cdot \norm{z - z'}_1
\end{align*}
where (\ref{eqn:hierarchy:mvt}) follows by the smoothness of $w_i$ and the mean value theorem, establishing that for some $\bar{z} \in \Delta_l$, $\card{w(z)_i - w(z')_i} = \langle \nabla w_i(\bar{z}), z-z' \rangle$.
Thus, to bound $\norm{\nabla w_i(z^*)}_\infty$ it suffices to bound the gradient over any $q \in \mP_{k-1}$ on $z^* \in \Delta_l$.
Using the degree and sparsity bounds, a simple convexity argument demonstrates that for any polynomial $q \in \mP_{k-1}$ and $i \in [l]$,
\[ \max_{z \in \Delta_l}\frac{\partial q}{\partial z_i}(z) \leq k -1. \]
In all, we derive the bound
\begin{gather*}
\norm{w(z) - w(z')}_\infty \le \max_{q \in \mP_{k-1}}\max_{z^* \in \Delta_l} \norm{\nabla q(z^*)}_\infty \cdot \norm{z - z'}_1 \le (k-1) \cdot \norm{z - z'}_1.
\end{gather*}
\end{proof}

Next we show that full multicalibration with a sufficiently small error parameter leads to smooth multicalibration. 

\begin{theorem}
\label{thm:int-to-smooth}
Fix a calibration class $\mC$ and an approximation parameter $\alpha \geq 0$.
Every  $(\mC, \beta, \delta)$-full multicalibrated predictor is $(\mC, \alpha)$-smooth multicalibrated for
    \[ \alpha =  \beta\cdot \ceil{1/\delta}^l + l\delta. \]
    Hence for $\delta \leq \alpha/l$ and $1/\delta \in \N$, we have $ \mcalreg_\delta(\alpha\delta^l - l\delta^{l+1}) \subseteq \mcals(\alpha)$.
\end{theorem}
To prove Theorem~\ref{thm:int-to-smooth}, we first establish a few key properties of calibration defined by weight functions that will be useful throughout our discussion.
We start by showing that $(\mC,\mW,\alpha)$-multicalibration is robust to small perturbations of the weight functions in $\mW$.
\begin{lemma}
\label{lem:robust2}
Let $v,w:\Delta^l \to [0,1]^l$ be weight functions such that $$\max_{p \in \Delta^l} \norm{v(p) - w(p)}_\infty \le \eta.$$
Then for any $l$-class predictor $f:\X \to \Delta^l$ and any $c:\X \to [0,1]$,
\begin{gather*}
\card{\E_\mD\lr{c(\x) \langle v(f(\x)), \y - f(\x) \rangle} - \E_\mD\lr{c(\x) \langle w(f(\x)), \y - f(\x) \rangle } } \le 2\eta.
\end{gather*}
\end{lemma}
\begin{proof}
Observe that for any $x \in \X$ and $y \in \Y$,
\begin{align*}
\card{c(x) \cdot \left( \langle v(f(x)) - w(f(x)) , y - f(x) \rangle \right)} 
&\le \card{c(x)} \cdot \norm{v(f(x)) - w(f(x))}_\infty\cdot \norm{y - f(x)}_1\\
&\le 2\eta
\end{align*}
since $|c(x)| \leq 1$ and $\norm{y-f(x)}_1 \le 2$.
The claim follows by averaging  over $(\x, \y) \sim \mD$ .
\end{proof}

A key application of the robustness to weight functions is that one can infer smooth multicalibration---a condition defined in terms of an infinite weight class---from multicalibration for a finite basis of functions that uniformly approximates every function in $\Lip$.  
\begin{definition}
    \label{def:basis}
    A collection of weight functions $\mW =\{w_i\}_{i=1}^k$ is a $(\eta, L)$ basis for $\Lip$ if for every $u \in \Lip$, there exists $v: \Delta_l \rgta [0,1]^l$ such that $\|u - v\|_\infty \leq \eta$ where
    \begin{align*} 
        v = \sum_{i=1}^k \lambda_i w_i,\ \sum_{i=1}^k|\lambda_i| \leq L.
    \end{align*}
\end{definition}
Typically, both $k$ and $L$ will be functions of $\eta$. Our interest in $\ell_\infty$ approximations is motivated by the following lemma. 

\begin{lemma}
\label{lem:smooth-mc}
If $\mW$ is $(\eta, L)$ basis for $\Lip$ and the predictor $f$ is $(\mC, \mW, \beta)$-multicalibrated on $\mC$, then $f$ is $(\mC, \alpha)$-smoothly multicalibrated where $\alpha = \beta L + 2\eta$.
\end{lemma}
\begin{proof}
For a weight function $u \in \mL_1$, let $v = \sum_i \lambda_iw_i$ be the $\ell_\infty$ approximation guaranteed from the definition of $\mW$. Then by linearity of expectation
\begin{align*} 
    \E_{\mD}[c(\x)\langle v(f(\x), (\y - f(\x))\rangle] &= \E_{\mD}\left[c(\x)\sum_{i=1}^n\lambda_i \langle w_i(f(\x)),  \y - f(\x) \rangle \right] \\
    &= \sum_{i=1}^n\lambda_i \E_{\mD}[c(\x) \langle w_i(f(\x)), \y - f(\x) \rangle]
\end{align*}
Taking absolute values and using the assumption that $f$ is $(\beta, \mW)$-multicalibrated on $\mC$ gives
\begin{align*} 
    \label{eq:robust2}    
    \abs{\E_{\mD}[c(\x)v (f(\x))(\y - f(\x))]} \leq \abs{\sum_{i=1}^n \lambda_i \beta} \leq \beta L
\end{align*}
To complete the proof, we now use Lemma \ref{lem:robust2} to conclude that $f$ is $\alpha = \beta L + 2\eta$-smoothly multicalibrated on $\mC$.
\end{proof}

\begin{proposition}
\label{prop:basis-l1}
    For any $l \geq 2$, the set $I^l_\delta$ is a $(l\delta/2, \ceil{1/\delta}^l)$ basis for $\Lip$.
\end{proposition}
\begin{proof}
    Recall that $\Pi^l_\delta$ partitions $[0,1]^l$ in a set of cubes (products of intervals). For each cube $\pi \in \Pi^l_\delta$, we pick the center point $z_\pi \in \pi$, note that it satisfies $\norm{z -x}_1 \leq \delta l/2$ for every $x \in \pi$. 
    Given a weight function $u \in \mL_1$, we approximate it by the function $v$ where
    \[ v(x) = \sum_{\pi \in \Pi^l_\delta} u(z_\pi)\ind{\pi}(x) \]
    which assigns the value $u(z_\pi)$ to every point $x \in \pi$. Since $u \in \mL_{1 \rgta \infty}$, for $x \in \pi$    
    \[ \norm{v(x) - u(x)}_\infty = \norm{u(x) - u(z_\pi)}_\infty \leq \norm{x - z_\pi}_1 \leq l \delta/2, \]
    \[ \sum_{\pi \in \Pi^l_\delta} |u(z_\pi)| \leq |\Pi^l_\delta| \leq \ceil{1/\delta}^l \]
    hence the claim follows. 
\end{proof}

Combining Proposition \ref{prop:basis-l1} with Lemma \ref{lem:smooth-mc} completes the proof of Theorem \ref{thm:int-to-smooth}. In the setting of binary labels, one gets the following strengthening, which improves the dependence on $\delta$, by observing that $|\Pi_\delta| \leq \ceil{1/\delta}$. 

\begin{proposition}
\label{prop:int-to-smooth}
For binary classification, let $\delta < \alpha$ and $1/\delta \in \N$. If $f$ is $(\mC, \alpha\delta - \delta^2,\delta)$-full multicalibrated then it is $(\mC, \alpha)$-smoothly multicalibrated, hence $\mcalreg_{\delta}(\alpha\delta - \delta^2) \subseteq \mcals(\alpha)$.
\end{proposition}

\eat{
While this shows that partition multicalibration enforces the most stringent constraints, we complement these containments with (partial) converse inclusions.
\begin{proposition}
Other direction.
\end{proposition}

Next, we show formally a hierarchy of inclusions from the weakest multiaccuracy, to low-degree multicalibration, to smooth multicalibration, to the strongest partition multicalibration.
Specifically, for each level we establish settings of the parameters such that the multicalibration guarantees of a given level imply the guarantees of the lower levels.

}

\subsection{Robustness of Low-Degree and Smooth Multicalibration}

While calibration is an intuitively desirable property of predictors, one frustration in reasoning about calibrated and multicalibrated predictors is that full (multi)calibration is not robust to small perturbations in predictions.
For instance, consider a binary prediction setting where $\X = \X_0 \cup \X_1$ is equally partitioned such that in one half all $\X_0$ have associated label $y=0$, and in the other half $\X_1$ all have $y=1$.
Then, the constant predictor $f(x) = 1/2$ is perfectly calibrated, but the $\eps$-close predictor $f(x) = 1/2-\eps$ for $x \in \X_0$ and $f(x) = 1/2+\eps$ for $x \in \X_1$, is far from calibrated.

In this section, we show that---in stark contrast to full multicalibration---low-degree and smooth multicalibration are both robust to small perturbations of the predictor.
\begin{theorem}
\label{thm:robust}
Let $f,g:\X \to \Delta^l$ be $l$-class predictors such that for $\delta > 0$,
\begin{gather*}
    \E\lr{\norm{f(\x) - g(\x)}_1} \le \delta.
\end{gather*}
\begin{itemize}
\item If $f$ is $(\mC,\alpha)$-degree-$k$-multicalibrated, then $g$ is $(\mC, \alpha + (2k-1)\delta)$-degree-$k$ multicalibrated.
\item If $f$ is ($\mC, \alpha)$-smoothly multicalibrated, then $g$ is  
$(\mC, \alpha + 3\delta)$-smoothly multicalibrated.
\end{itemize}
\end{theorem}

This theorem is a consequence of the following more general lemma which applies to any weight family $\mW$ with bounded $\ell_1 \rgta \ell_\infty$ Lipschitz constant.
\begin{lemma}
\label{lem:robust}
Let $f,g:\X \to \Delta^l$ be $l$-class predictors such that for $\delta > 0$,
\begin{gather*}
    \E\lr{\norm{f(\x) - g(\x)}_1} \le \delta.
\end{gather*}
Let $\mW \subseteq r\mL_{1 \rgta \infty}$. 
If $f$ is $(\mC,\mW,\alpha)$-multicalibrated, then $g$ is $(\mC, \mW, \alpha + (2r + 1)\delta)$-multicalibrated.
\end{lemma}
\begin{proof}
Fix any $w \in r\Lip$. By the Lipschitz property, we derive the following inequalities:
\begin{align*}
\card{\langle w(f(x)) - w(g(x)), y \rangle} \le \norm{w(f(x)) - w(g(x))}_\infty
\le r\norm{f(x) - g(x)}_1 
\end{align*}
and
\begin{align*}
&\card{\langle w(f(x)), f(x) \rangle - \langle w(g(x)), g(x) \rangle}\\
& = \card{\langle w(f(x)), f(x) -g(x) \rangle  + 
\langle w(f(x)) - w(g(x)), g(x) \rangle}\\ 
& \leq \norm{w(f(x)}_\infty  \cdot \norm{f(x) - g(x)}_1 + \norm{w(f(x)) - w(g(x))}_\infty \cdot \norm{g(x)}_1\\
&\le (r+1) \norm{f(x) - g(x)}_1.
\end{align*}
Thus, for every $x \in \X$, by the triangle inequality, we bound the difference of the expressions by 
\begin{gather*}
\card{\langle w(f(x)), y - f(x) \rangle - \langle w(g(x)), y - g(x) \rangle} \le (2r +1) \norm{f(x) - g(x)}_1 
\end{gather*}
Fix $c \in \mC$ and consider the difference for $f$ and $g$ in the smooth multicalibration constraint for $c$.
\begin{align*}
&\card{\E_\mD\lr{c(\x) \langle w(f(\x)), \y - f(\x) \rangle} - \E_\mD\lr{c(\x) \langle w(g(\x)), \y - g(\x)\rangle }}\\
&= \card{ \E_\mD\lr{c(\x) \left(\langle w(f(\x)), \y - f(\x) \rangle - \langle w(g(\x)), \y - g(\x)\rangle\right) } }\\
&\le \max_{x \in \X} \card{c(x)} \cdot \card{\E_\mD\lr{\langle w(f(\x)), \y - f(\x) \rangle - \langle w(g(\x)), \y - g(\x)\rangle}}\\
&\le (2r +1)\E_\mD\lr{\norm{f(\x) - g(\x)}_1}\\
&= (2r +1)\delta
\end{align*}
where we use the fact that $\card{c(x)} \le 1$ for all $c \in \mC$.
Since by assumption $f$ is $(\mC,\mW,\alpha)$ multicalibrated, we get
\begin{align*}
\card{\E_\mD\lr{c(\x) \langle w(g(\x)), \y - g(\x) \rangle}} &\leq  \card{\E_\mD\lr{c(\x) \langle w(f(\x)), \y - f(\x)\rangle }} + (2r +1)\delta\\
&\leq \alpha + (2r +1)\delta.
\end{align*}
This holds for all $c \in \mC, w \in \mW$, hence $g$ is $(\mC, \mW, \alpha + (2r + 1)\delta)$-multicalibrated.
\end{proof} 
\section{Moment Sandwiching from Low-Degree Multicalibration}
\label{sec:properties}

In this section, we prove the key technical results, establishing moment sandwiching bounds for low-degree multicalibrated predictors.
We begin by establishing the general theorem, then show various corollaries, for bounds on the confusion probabilities in multi-class prediction, and the correlation between the predicted values and the true outcomes.

We begin with some useful notational shorthand.
Recall that  $\mC =\{ c:\X \rgta [0,1]\}$. We denote the measure of $c$ under $\mD$ by $\mu_c = \E_{\mD}[c(\x)]$.  Define the distribution $\mD_c$ over $\X \times \zo$ obtained by conditioning on $c$ by   
\begin{align*}
    \mD_c(x, y) = \frac{c(x)\mD(x, y)}{\mu_c} 
\end{align*}
 For ease of notation, we will sometimes use $\alpha_c = \alpha/\mu_c$. 
\eat{
Recall the definitions of variance and covariance
\begin{align*} 
\Cov_\mD[\z_1, \z_2] &= \E_\mD[\z_1\z_2] - \E_{\mD}[\z_1]\E_{\mD}[\z_2] = \E_{\mD}\lt[\z_1(\z_2 - \E_{\mD}[\z_2])\rt] \\
\Cov_\mD [\z, \z] &= \Var_\mD[\z] = \E_\mD[\z^2] - \E_\mD[\z]^2.
\end{align*}
}

\begin{lemma}
\label{lem:level-j}
    Let $k \geq 2$ and $f\in \mcal_k(\alpha)$. For every degree $d \in [k]$, label $\ell \in [l]$ and $c \in \mC$, 
    \begin{align}
    \label{eq:level-j}
        \abs{\E_{\mD_c}[f_\ell(\x)^{d - 1}f^*_\ell(\x)] - \E_{\mD_c}[f_\ell(\x)^{d}]} \leq \frac{\alpha}{\mu_c} = \alpha_c.
    \end{align}    
\end{lemma}
\begin{proof}
    We consider the function $w : \Delta^l \rgta [0,1]^l$ where $w_\ell(f) = f_\ell^{d-1}$ and $w_j(f) = 0$ for $j \neq \ell$. It is easy to see that $w \in \mP_{k-1}$. For this choice of $w$,
    degree $k$ multicalibration implies 
    \begin{align*}
        \card{\E_{\mD}[c(\x) f_\ell(\x)^{d-1}(\y_\ell - f_\ell(\x))]} \leq \alpha.
    \end{align*}
    Switching to the conditional distribution $\mD_c$, we can rewrite this as
    \begin{align*}
        \mu_c \card{\E_{\mD_c}[ f_\ell(\x)^{d-1}(f^*_\ell(\x)  - f_\ell(\x))]} \leq \alpha
    \end{align*}
    since $\E_[\y_\ell|x] = f^*(x)$. Dividing both sides by $\mu_c$ and rearranging gives the desired bound. 
\end{proof}

Note that this guarantee is meaningful only if $\mu_c$ is larger than $\alpha$. This is to be expected, since we cannot hope to get strong conditional guarantees for sets that are very small.

Our goal is to show the following sandwiching bound which we state for the setting $\alpha =0$ for clarity. For all $d \in [k], \ell \in [l], c \in \mC$,:
\[  \E_{\mD_c}[f^*_\ell(\x)^d] \geq
 \E_{\mD_c}[f_\ell(\x)^d]   = \E_{\mD_c}[f_\ell(\x)^{d - 1}f^*_\ell(\x)]
\geq \E_{\mD_c}[f_\ell(\x)^{d -1}]\E_{\mD_c}[f^*_\ell(\x)]. \]
The middle equality is by Lemma \ref{lem:level-j}, which is immediate from the definition of degree-$k$ multicalibration. The key ingredients are the outer inequalities. The lower bound can be interpreted as saying that $f_\ell^{d -1}$ is positively correlated with the label being $\ell$, since
\[ \Cov_{\mD_c}[f_\ell(\x)^{d - 1}, \y_\ell] = \E_{\mD_c}[f_\ell(\x)^{d - 1}f^*_\ell(\x)] - \E_{\mD_c}[f_\ell(\x)^{d -1}]\E_{\mD_c}[f^*_\ell(\x)]   \geq 0. \] 
The upper bound can be seen as an upper bound {\em in lieu of exact moment matching}, it says that the first $k$ moments of  $f_\ell$ are dominated by those of the ground truth $f_\ell^*$ for every $\ell \in [l]$.

When $\alpha > 0$, the overall form of the inequalities stays the same, with some slack depending on $\alpha_c$. But the two terms in the middle are only approximately equal.  This makes it easier to state the bounds separately, which we do in Theorem \ref{thm:main-tech} and Corollary \ref{cor:main-tech}. 

\begin{theorem}[Formal restatement of Theorem~\ref{thm:informal:main}]
\label{thm:main-tech}
Fix $k \geq 2$ and a predictor $f:\X \rgta \Delta_l$ where  $f\in \mcal_k(\alpha)$. For every degree $d \in [k]$, label $\ell \in [l]$ and $c \in \mC$, the following sandwiching bound holds:
\begin{align}
    \label{eq:moment-sw1}
 d\frac{ \alpha}{\mu_c} + \E_{\mD_c} [f^*_\ell(\x)^d]\geq  \E_{\mD_c} [f_\ell(\x)^d]  \geq \E_{\mD_c} [f^*_\ell(\x)]\E_{\mD_c}[f_\ell(\x)^{d-1}]  - \frac{\alpha}{\mu_c}.
\end{align}
\end{theorem}
\begin{proof} 
We first prove Equation \eqref{eq:moment-sw1}, starting from the upper bound on $\E[f_\ell(\x)^d]$. We claim that
\begin{align*}
\E_{\mD_c}[f_\ell(\x)^{d}]^\frac{d -1}{d}\E_{\mD_c}[f^*_\ell(\x)^{d}]^\frac{1}{d} \geq    \E_{\mD_c} [f_\ell(\x)^{d -1} f^*_\ell(\x)] \geq \E_{\mD_c}[f_\ell(\x)^{d}] - \alpha_c
\end{align*}
The  the upper bound is by Holder's inequality whereas the lower bound is by Equation \eqref{eq:level-j}. Dropping the term in the middle,
\begin{align*}
    \E_{\mD_c}[f_\ell(\x)^{d}]^\frac{d -1}{d}\E_{\mD_c}[f^*_\ell(\x)^{d}]^\frac{1}{d} &\geq \E_{\mD_c}[f_\ell(\x)^{d}] - \alpha_c
\end{align*}
We may assume $f$ is not identically $0$, since otherwise the upper bound is trivial, and hence that its moments are positive. Divide both sides by $\E_{\mD_c}[f_\ell(\x)^{d}]$ to get
\begin{align*}
    \left(\frac{\E_{\mD_c}[f^*_\ell(\x)^{d}]}{\E_{\mD_c}[f_\ell(\x)^{d}]}\right)^\frac{1}{d} & \geq 
\left( 1 - \frac{\alpha_c}{\E_{\mD_c}[f_\ell(\x)^{d}]}\right)
\end{align*}
Raising both sides to the power $d$, and using $(1 -\eta)^d \geq 1 - d \eta$ for $k \in \Z^+$ and $\eta > 0$
\begin{align*}
    \frac{\E_{\mD_c}[f^*_\ell(\x)^{d}]}{\E_{\mD_c}[f_\ell(\x)^{d}]} \geq 
    \left(1 - \frac{\alpha_c}{\E_{\mD_c}[f_\ell(\x)^{d}]}\right)^d \geq \left(1 - d\frac{ \alpha_c}{\E_{\mD_c}[f_\ell(\x)^{d}]}\right)\\
\end{align*}
Multiplying throughout by $\E_{\mD_c}[f_\ell(\x)^{d}]$
\begin{align}
\label{eq:better}
    \E_{\mD_c}[f^*_\ell(\x)^d] \geq \E_{\mD_c}[f_\ell(\x)^d] - d\alpha_c
\end{align}
which gives the desired upper bound for $d \in \{1, \ldots, k\}$. 

We now prove the lower bound on $\E[f^d]$ in Equation \eqref{eq:moment-sw1}. By multiaccuracy and convexity, 
\begin{align*}
\label{eq:better2}
    \E_{\mD_c}[f^*_\ell(\x)] &\leq \E_{\mD_c}[f_\ell(\x)] + \alpha_c \leq \E_{\mD_c}[f_\ell(\x)^d]^\fr{d} + \alpha_c\\ 
    \E_{\mD_c}[f_\ell(\x)^{d -1}] &\leq \E_{\mD_c}[f_\ell(\x)^{d}]^\frac{d-1}{d}
\end{align*}
Multiplying these together gives
\begin{align*}
    \E_{\mD_c}[f_\ell(\x)^{d-  1}]\E[f^*_\ell(\x)] \leq  (\E_{\mD_c}[f_\ell(\x)^d]^\fr{d} + \alpha_c)\E_{\mD_c}[f_\ell(\x)^d]^{\frac{d-1}{d}} \leq \E_{\mD_c}[f_\ell(\x)^d] + \alpha_c
\end{align*}
where the last inequality uses the fact that $f_\ell\in [0,1]$. 
This completes the proof of Equation \eqref{eq:moment-sw1}. 
\end{proof}

\begin{corollary}
\label{cor:main-tech}
Under the conditions of Theorem \ref{thm:main-tech},
 the following sandwiching bound holds:
\begin{eqnarray}
\label{eq:moment-sw2}
(d + 1)\frac{ \alpha}{\mu_c} + \E_{\mD_c} [f^*_\ell(\x)^d]  \geq & \E_{\mD_c} [f_\ell(\x)^{d -1}f^*_\ell(\x)] & \geq \E_{\mD_c} [f^*_\ell(\x)]\E_{\mD_c}[f_\ell(\x)^{d-1}]  - 2\frac{\alpha}{\mu_c}.
\end{eqnarray} 
\end{corollary}
\begin{proof}
    We start from the bound on $\E[f^d]$ in Equation \eqref{eq:moment-sw1} and use Equation \eqref{eq:level-j} which implies that the bounds hold for  $\E[f^{d -1}f^*]$  with an additional slack of $\alpha_c$. 
\end{proof}

\subsection{Bounding the True Positive Rates}

Imagine a binary prediction problem on a population $P = A \cup B$ where one half $A$ has a $70\%$ chance of having outcome $\y =1$, and the other half $B$ has only a $30\%$ chance.
When we consider this population in isolation, multiaccuracy is a very weak condition.
In particular, the predictor $g:\X \to [0,1]$ that predicts $g(\x) = 1$ on $A$, and $g(\x) = 0$ on $B$ satisfies accuracy-in-expectation over $P$.
This predictor is over-confident in its predictions; it does not give an accurate sense of the uncertainty in its predictions.
A key property of calibration is that it disallows such overconfidence, by explicitly requiring that when $f(\x) \approx v$, then $\E[\y \vert f(\x)] \approx v$.  

A different way to quantify confidence is to look at $\E[f(\x)|\y =1]$ and $\E[f(\x)|\y =0]$. The latter has been termed  the {\em generalized false positive rate}, so we refer to the former as the {\em generalized true positive rate}.
Common sense suggests that we want a low false positive rate, and a high true positive rate.
But how low a false positive rate is desirable?
A quick calculation shows that the predictor $g$ above has a generalized false positive rate of $0.3$, whereas even the Bayes optimal predictor $f^*$ has a higher false positive rate of $0.58$!
In general, only seeking to minimize the false positive rate and false negative rate risks preferring predictors that overstate their confidence, since it does not incentivize the predictor to convey the level of uncertainty in its predictions.

We give bounds on the true positive rates of any predictor in $\mcal_2$. 

\begin{definition}
    \label{def:tp}
    For a predictor $f:\X \rgta \Delta_l$, label $\ell \in [l]$ and $c \in \mC$, define the true positive rate for $f$ on $\ell$ conditioned on $c \in \mC$ to be $\tau_c(f,\ell) = \E_{\mD_c}[f_\ell(\x)|\y_\ell =1]$.
\end{definition}

\begin{lemma}
\label{lem:tpr}
    Every predictor $f \in \mcal_2$ satisfies
    \begin{align*}
        \frac{3\alpha}{\E_{\mD_c}[f_\ell^*(\x)]\mu_c} + \tau_c(f^*, \ell) \geq \tau_c(f, \ell) \geq \E_{\mD_c}[f_\ell(\x)] -  \frac{2\alpha}{\E_{\mD_c}[f_\ell^*(\x)]\mu_c}    
        \end{align*}
\end{lemma}

The lower bound says the true positive rate for label $\ell$ is at least as much as the overall positive rate for that label in the population $\mC$. This is equivalent to positive correlation between $f_\ell(\x)$ and $\y_\ell$ conditioned on $c$. The upper bound asserts that the correlation is not exaggerated, being upper bounded by the true positive rate of the Bayes optimal predictor. Our proof relies on the following characterization of true positive rates:

\begin{lemma}
\label{lem:eq-tpr}
    We have
    \[ 
    \tau_c(f, \ell) = \frac{\E_{\mD_c}[f_\ell(\x)f^*_\ell(x)]}{\E_{\mD_c}[f^*_\ell(\x)]} \]
\end{lemma}
\begin{proof}
    By the definition of $\tau_c$, 
    \begin{align*}
        \tau_c(f, \ell)  &= \E_{\mD_c}[f_\ell(\x)|\y_\ell =1]\\
        &= \frac{\sum_{x} \mD_c(\x =x \wedge \y_\ell =1)f_\ell(x)}{\sum_{x} \mD_c(\x =x \wedge \y_\ell =1)}\\
        & = \frac{\sum_{x} \mD_c(\x =x) f_\ell^*(x)f_\ell(x)}{\sum_{x} \mD_c(\x =x)f_\ell^*(x)}\\
        &= \frac{\E_{\mD_c}[f_\ell(\x)f_\ell^*(\x)]}{\E_{\mD_c}[f_\ell^*(\x)]}.
    \end{align*}
\end{proof}

\begin{proof}[Proof of Lemma \ref{lem:tpr}]
Equation \eqref{eq:moment-sw2} for $d = 2$ gives 
\begin{eqnarray}
\label{eq:moment-sw2.2}
3\frac{ \alpha}{\mu_c} + \E_{\mD_c} [f^*_\ell(\x)^2]  \geq & \E_{\mD_c} [f_\ell(\x)f^*_\ell(\x)] & \geq \E_{\mD_c} [f^*_\ell(\x)]\E_{\mD_c}[f_\ell(\x)]  - 2\frac{\alpha}{\mu_c}.
\end{eqnarray} 
The claimed bounds follow by dividing  throughout by $\E[f_\ell^*]$ and observing that by Lemma \ref{lem:eq-tpr}, 
\begin{align*} 
\tau_c(f, \ell) &= \frac{\E_{\mD_c} [f^*_\ell(\x)f_\ell(\x)]}{\E_{\mD_c}[f^*_\ell(\x)]}\\
\tau_c(f^*, \ell) &= \frac{\E_{\mD_c} [f^*_\ell(\x)^2]}{\E_{\mD_c}[f^*_\ell(\x)]}.
\end{align*}

\end{proof}

We can define higher moment analogues of the true positive rate by considering the quantity $\E[f_\ell(\x)^d|\y_\ell = 1]$ for $d \geq 2$; for which we prove the following bound:

\begin{lemma}
    For every $d \in [k-1]$ and $c \in \mC$, every $f \in \mcal_k$ satisfies the following bounds :
    \[ \frac{(d +1)\alpha}{\E_{\mD_c}[f_\ell^*(\x)]\mu_c} + \E_{\mD_c}[f_\ell^*(\x)^d|\y_\ell = 1] \geq \E_{\mD_c}[f_\ell(\x)^d|\y_\ell = 1] \geq \E_{\mD_c}[f_\ell(\x)^d] -\frac{2\alpha}{\E_{\mD_c}[f_\ell^*(\x)]\mu_c}\] 
\end{lemma}

\subsection{Bounds for the Confusion Matrix}

We generalize the discussion of error rates to the multi-class setting.
Here, we are concerned with the confusion matrix, whose $ij$th entry corresponds to the probability of predicting $j$ when the true label is $i$.
In this setting, we think about sampling a predicted label $\z$ according to the prediction $f(\x)$; this bears similarity to the model of outcome indistinguishability.
With some manipulation, it is not hard to see that the collision and confusion rates can be audited directly within the OI framework \cite{OI,dwork2022beyond} to test the closeness of the matrices as in Lemma~\ref{lem:max-norm}.

Given a distribution $\mD$ on $\X \times [l]$ and a predictor $f:\X \rgta \Delta_l$, let us assume that the predictor $f$ generates a label $\z \in [l]$ where $\Pr[\z = j|x] = f_j(x)$. The $l \times l$ confusion matrix has entries $b_{ij} = \Pr_\mD[\y = i \wedge \z = j]$ for $i, j \in [l]$. It is easy to see that 
\[ \Pr_\mD[\y = i \wedge \z = j] = \E_{\mD}[f_i^*(\x)f_j(\x)]. \]
Associating points in $\Delta_l$ with column vectors in $\R^l$,and using $f^T$ to denote the transpose of $f$, we can define the confusion matrix as
\[ B_\mD(f^*, f) = \E_{\mD}[f^*(\x)f(\x)^T].  \] 
Define the max norm of a matrix $A = \{a_{ij}\}_{i, j \in [l]}$ by $\norm{A}_{\max} = \max_{i, j \in [l]}|a_{ij}|$. 

\begin{lemma}
\label{lem:max-norm}
    Let $f \in \mcal_2(\alpha)$ and $c \in \mC$. Then we have
    \[  \norm{B_{\mD_c}(f^*, f) - B_{\mD_c}(f, f)}_{\max} \le \frac{\alpha}{\mu_c}.\]
\end{lemma}
\begin{proof}
    For $i, j \in [l]$, define the function $w:\Delta \rgta [0,1]^l$ by $w_i(f) = f_j$ and $w_{i'}(f) = 0$ for $i' \neq i$. Clearly $w \in \mP_1$, hence by the definition of degree-$2$ multicalibration applied to $w$, 
    \[\mu_c\card{\E_{\mD_c}[f_j(\x)(\y_i - f_i(\x)]} \leq \alpha. \]
    We can rewrite this condition as
    \[ \card{\E_{\mD_c}[f^*_i(\x)f_j(\x)] - \E_{\mD_c}[f_i(\x)f_j(\x)]} \leq \alpha_c \]
    which implies the claim.
\end{proof}

Next we show that the confusion matrix for $f$ is dominated by that for $f^*$ in the PSD order.
\begin{lemma}
\label{lem:psd-order}
For any $f \in \mcal_2(\alpha)$ and $c \in \mC$, we have
\[ B(f,f) \preccurlyeq B(f^*, f^*) + 2l\alpha \cdot I_{l \times l}.\]
\end{lemma}
\begin{proof}
Fix a unit vector $u \in \R^\ell$. By the definition of the PSD order, it suffices to show that
\begin{align*}
    \E_{\mD_c}[(u^Tf(\x))^2]  \leq     \E_{\mD_c}[(u^Tf^*(\x))^2] + 2l \alpha.
\end{align*}

We have
\begin{align*} 
\E[(u^T f(\x))^2]  &= \sum_{ij}u_iu_j\E_{\mD_c}[f_i(\x)f_j(\x)]\\
& \leq \sum_{ij}u_iu_j\E_{\mD_c}[f^*_i(\x)f_j(\x)] + \alpha \sum_{ij}u_iu_j\\
& \leq \E_{\mD_c}[(u^tf^*(\x))(u^tf(\x))] + l\alpha \\
& \leq \E_{\mD_c}[(u^tf^*(\x))^2]^{1/2}]\E_{\mD_c}[(u^tf(\x))^2]^{1/2} + l\alpha 
\end{align*}
Dividing both sides by  $\E_{\mD_c}[(u^Tf(\x))^2]$ we get
\begin{align*}
    \frac{\E_{\mD_c}[(u^Tf^*(\x))^2]^{1/2}}{\E_{\mD_c}[(u^Tf(\x))^2]^{1/2}}\geq 1 - \frac{l \alpha}{\E_{\mD_c}[(u^Tf(\x))^2]}
\end{align*}
Squaring and using $(1- \eta)^2 \geq  1 -2\eta$ gives
\begin{align*}
    \E_{\mD_c}[(u^Tf^*(\x))^2] \geq \E_{\mD_c}[(u^Tf(\x))^2]  - 2l \alpha,
\end{align*}
which implies the desired bound.
\end{proof}

As an application, we can generalize our bounds on true positive rates to allow combinations of labels. As motivation, consider a model trained to assign images from a large label set $[l]$. We might want to know how well the model does on the set of cat images, where cats correspond to a subset $L \subseteq [l]$. In analogy with the true positive rate on label, we could measure the probability that the image belongs to $L$, and the predicted label is also in $L$, without distinguishing between labels within $L$. This motivates our next definition. We use the notation $f_L(x) = \sum_{\ell \in L}f_\ell(x)$.

\begin{definition}
For a predictor $f:\X \rgta \Delta_l$, set $L \subseteq [l]$ and $c \in \mC$, define the true positive rate for $f$ on $L$ conditioned on $c \in \mC$ to be 
    \[ \tau_c(f,L) = \E_{\mD_c}\Big[f_L(\x)|\sum_{l \in L}\y_\ell =1\Big].\]
\end{definition}

We present the following generalization of Lemma \ref{lem:tpr}:
\begin{lemma}
\label{lem:set-tpr}
    For a subset $L \subseteq L$ of labels, every predictor $f \in \mcal_2$ satisfies
    \begin{align*}
        \frac{2l|L|^2\alpha}{\E_{\mD_c}[f_L(\x)]\mu_c} + \tau_c(f^*, L) \geq \tau_c(f, L) \geq \E_{\mD_c}[f_L(\x)] -  \frac{(|L|^2+|L|)\alpha}{\E_{\mD_c}[f_L^*(\x)]\mu_c}. 
        \end{align*}
\end{lemma}
\begin{proof}
    In analogy with Lemma \ref{lem:eq-tpr}, we can show that
    \[ 
    \tau_c(f, L) = \frac{\E_{\mD_c}[f_L(\x)f^*_L(x)]}{\E_{\mD_c}[f^*_L(\x)]} = \frac{\ind{L}^TB_{\mD_c}(f^*, f)\ind{L}}{\E_{\mD_c}[f^*_L(\x)]}. \]
    Using Lemma \ref{lem:max-norm}, we have
    \begin{align*} 
        \ind{L}^TB_{\mD_c}(f^*, f)\ind{L} &\geq \ind{L}^TB_{\mD_c}(f, f)\ind{L} - |L|^2\alpha_c \\
        & = \E_{\mD_c}[f_L(\x)^2] - |L|^2\alpha_c\\
        & \geq \E_{\mD_c}[f_L(\x)]^2 - |L|^2\alpha_c\\
        & \geq \E_{\mD_c}[f_L(\x)] \E_{\mD_c}[f^*_L(\x)] - (|L|^2 +|L|)\alpha_c
    \end{align*}
    where the last line used the following consequence of multiaccuracy 
    \[ \card{\E_{\mD_c}[f_L(\x) - f^*_L(\x)]} \leq |L|\alpha_c \]
    Lemma \ref{lem:psd-order} implies an upper bound of
    \begin{align*} 
    \ind{L}^TB_{\mD_c}(f^*, f)\ind{L} &\leq \ind{L}^TB_{\mD_c}(f^*, f^*)\ind{L} +2l |L|^2\alpha_c\\
    &\leq  \E_{\mD_c}[f^*_L(\x)^2] +2l |L|^2\alpha+c.
    \end{align*}
    The claim now follows by dividing throughout by $\E_{\mD_c}[f^*_L(\x)]$
\end{proof}

\subsection{Covariance Guarantees for Degree-2 Multicalibration}

We present a detailed analysis of the covariance between the predictions $f_\ell(\x)$ and label being $\ell$ conditioned on $c$.  As one might expect, the ground truth predictor is indeed positively correlated with the labels. Multiaccuracy guarantees that the expectations of $f_\ell(\x)$ and $\y_\ell$  are equal conditioned on $c$, but this need not imply that they are positively correlated. Our main result of this section is $\mcal_2$ guarantees positive correlation, and that the correlation is conservative compared with the ground truth predictor. We also provide an example showing that the correlation can indeed be negative for $f \in \mcal_1$.

\begin{theorem}
\label{thm:corr}
For $f \in \mcal_2(\alpha)$ and $c \in \mC$, we have
\begin{align}
\label{eq:cov-ineq}
\Var_{\mD_c}[f_\ell(\x)] - 2\alpha_c \leq \Cov_{\mD_c}[f_\ell(\x), \y_\ell] \leq \Cov_{\mD_c}[f^*_\ell(\x), \y_\ell] + 6\alpha_c. 
\end{align}
\end{theorem}

Our key technical lemma uses these to show that conditioned on any $c \in \mC$, every predictor in $\mcal_2$ has variance not much larger than the Bayes optimal predictor. Its proof will use the following properties every $f \in \mcal_2(\alpha)$ satisfies by Equation \eqref{eq:level-j}.

\begin{align}
    \abs{\E_{\mD_c}[f_\ell(\x)] - \E_{\mD_c}[\y]} \leq \alpha_c, \label{eq:deg0}\\
    \abs{\E_{\mD_c}[f_\ell(\x)^2] - \E_{\mD_c}[f_\ell(\x)f^*_\ell(\x)} \leq \alpha_c. \label{eq:deg1}
\end{align}
\begin{lemma}
\label{lem:low-var}
For $f \in \mcal_2$ and $c \in \mC_1$,
\begin{align}
        \Var_{\mD_c}[f_\ell(\x)] &\leq \Var_{\mD_c}[f^*_\ell(\x)] + 4\alpha_c. \label{eq:var}\\
        \label{eq:cov-f*}
        \Cov_{\mD_c}[f^*_\ell(\x)\y] &= \Var_{\mD_c}[f^*_\ell(\x)].\\
    \label{eq:cov-f}
        \abs{\Cov_{\mD_c}[f_\ell(\x)\y] - \Var_{\mD_c}[f_\ell(\x)]} & \leq 2\alpha_c.
\end{align}
\end{lemma}
\begin{proof}
We observe that by Equation \eqref{eq:better} with $d =2$, 
\begin{align}
\label{eq:2-norm}
    \E_{\mD_c}[f^*_\ell(\x)^2] \geq \E_{\mD_c}[f_\ell(\x)^2] - 2 \alpha_c.  
\end{align}

By Equation \eqref{eq:deg0} we have
\begin{align*}
    \abs{\E_{\mD_c}[f_\ell(\x)]^2 - \E_{\mD_c}[f^*_\ell(\x)^2]} \leq \abs{\E_{\mD_c}[f_\ell(\x)] + \E_{\mD_c}[f^*_\ell(x)]} \abs{\E_{\mD_c}[f_\ell(\x)] - \E_{\mD_c}[f^*_\ell(x)]} \leq 2\alpha_c. 
\end{align*}
Applying this together with Equation \eqref{eq:2-norm} gives
    \begin{align*}
            \Var_{\mD_C}[f^*_\ell(\x)] - \Var_{\mD_c}[f_\ell(\x)] &\geq \E_{\mD_c}[f^*_\ell(\x)^2] - \E_{\mD_C}[f_\ell(\x)^2] + \E_{\mD_C}[f_\ell(\x)]^2 - \E_{\mD_c}[f^*_\ell(\x)]^2 \\
            & \geq -2\alpha_c -2\alpha_c = -4\alpha_c
    \end{align*}
    which implies the desired bound. 

By definition of $f^*$, $\E[\y|\x] = f^*_\ell(\x)$. Hence
\begin{align}
    \label{eq:eq-f*-corr}
\Cov_{\mD_c}[f^*_\ell(\x)\y] = \E_{\mD_c}[f^*_\ell(\x)\y] - \E_{\mD_c}[f^*_\ell(\x)]\E_\mD[\y] = 
\E_{\mD_c}[f^*_\ell(\x)^2] - \E_{\mD_c}[f^*_\ell(\x)]^2 = \Var_{\mD_c}[f^*_\ell(\x)].
\end{align} 
    
By Equation \eqref{eq:deg1}, 
    \begin{align}
    \label{eq:lb1}
        \left| \E[f_\ell(\x)\y] - \E[f_\ell(\x)^2]\right| \leq \alpha_c
    \end{align}
    Using Equation \eqref{eq:deg0} we have
    \begin{align}
    \label{eq:lb2}
         \abs{\E[f_\ell(\x)]\E[\y] - \E[f_\ell(\x)]^2 } \leq \E[f_\ell(\x)\abs{\E[f_\ell(\x) -\y]} \leq \alpha_c \E[f_\ell(\x)] \leq \alpha_c.
    \end{align}
Finally, using the definitions of variance and covariance, we have     
\begin{align*}
        \abs{\Cov[f_\ell(\x), \y] - \Var[f_\ell(\x)]} \leq  \abs{\E[f_\ell(\x)\y] - \E[f_\ell(\x)^2]} + \abs{\E[f_\ell(\x)]\E[\y] - \E[f_\ell(\x)]^2 } \leq 2\alpha_c
\end{align*}
    where we use Equations \eqref{eq:lb1} and \eqref{eq:lb2}.
\end{proof}

\begin{proof}[Proof of Theorem \ref{thm:main-tech}]
The lower bound is an immediate consequence of Equation \eqref{eq:cov-f}. The upper bound follows since
\begin{align*}
    \Cov_{\mD_c}[f_\ell(\x), y] \leq \Var_{\mD_c}[f_\ell(\x)] + 2\alpha_c \leq \Var_{\mD_c}[f^*_\ell(x)] + 6\alpha_c = \Cov_{\mD_c}[f^*_\ell(x), y] + 6\alpha_c
\end{align*}
where we use Equations \eqref{eq:cov-f}, \eqref{eq:var} and \eqref{eq:cov-f*} respectively.
\end{proof}

We complement this by an example  showing that the correlation between $f_\ell(\x)$ and $\y_\ell$ can be negative for $f \in \macc =\mcal_1$. In particular, we show that this is true even when $\y_\ell$ is obtained by standard methods such as least-squres or logistic regression (on linear combinations from $\mC$).

\begin{lemma}
There exist a distribution $\mD$ on $\zo^2 \times \zo$, a constraint set $\mC$, a constraint $c$, and a label $\ell \in \zo$, such that if $f \in \macc$ is obtained by least-squares or logistic regression, then 
\[\Cov_{D_c}[f_\ell(\x), \y_\ell] = -1/12 < 0.\]
\end{lemma}

The example we provide is in the binary classification setting, with $\mD = \{0, 1\}^2$ and the sets $\mC$ all being edges (that is, the sets $x_1=0$, $x_1=1$, $x_2=0$, and $x_2=1$). The covariance we achieve is $-1/12$, though we note that it can be made arbitrarily close to $-1/4$, the minimum possible covariance of two $[0, 1]$ random variables. The full example is provided in Appendix~\ref{sec:correlation-counterexample}.
 
\section{The Complexity of Low-Degree Multicalibration}
\label{sec:alg}
\newcommand{\WAL}{\mathsf{WAL}}

Here, we establish upper bounds on the time and sample complexity for obtaining low-degree multicalibration.
We begin by describing a completely generic multicalibration algorithm that works for any class of weight functions $\mW$ and for any number of class labels $k \in \N$.
Importantly, following \cite{hkrr2018,kgz}, the algorithm reduces the task of learning a weighted multicalibrated predictor to the task of \emph{weak agnostic learning} the class $\mC$.
We analyze the algorithm in terms of its oracle-efficiency, assuming access to a weak agnostic learner.

With upper bounds on the complexity of learning a multicalibrated predictor using a generic weight class $\mW$, we instantiate the bounds for the low-degree, smooth, and indicator variants of multicalibration.
We show that, for meaningful settings of the parameters, low-degree multicalibration is considerably more sample efficient than the original formulation of multicalibration.
This effect is particularly pronounced as the number of class labels $l$ grows.

\subsection{Learning Weighted Multicalibrated Predictors}

In Algorithm~\ref{alg:wmc}, we describe a procedure for learning multicalibrated predictors.
The algorithm assumes oracle-access to a weak agnostic learner \cite{KalaiMV08,feldman2009distribution}.
\begin{definition}[Weak Agnostic Learning]
For a data distribution $\mD$ supported on $\X \times [-1,1]$, a weak agnostic learner for a hypothesis class $\mC \subseteq \set{c:\X \to [0,1]}$ is a learning procedure that takes labeled data $D = \set{(x_i,z_i)}_{i=1}^m$, where each sample $(\x,\z) \sim \mD$.
For $\alpha > 0$, the learning procedure returns an element of $\mC \cup \set{\bot}$
\begin{gather*}
    c \gets \WAL_{\mC,\alpha}(D)
\end{gather*}
satisfying the following properties.
\begin{enumerate}[(1)]
    \item if there exists $c' \in \mC$ such that $\E_\mD\lr{c'(\x) \cdot \z} \not \in [-\alpha,\alpha]$, then $c \neq \bot$ and $\E_\mD\lr{c(\x) \cdot \z} \ge \alpha/2$.
    \item if $c = \bot$, then for all $c' \in \mC$, $\E_\mD\lr{c'(\x) \cdot \z} \in [-\alpha, \alpha]$.
\end{enumerate}
We say the sample complexity of the weak agnostic learner, $m = m(\mC,\alpha,\beta)$, is the number of samples from $\mD$ necessary to guarantee properties (1) and (2) with probability at least $1-\beta$.
\end{definition}
Intuitively, a weak agnostic learner searches for some $c \in \mC$ that correlates nontrivially with the labels given by $z$.\footnote{Our analysis does not assume that $\WAL$ is a proper learner.  All of the results hold equally for improper weak agnostic learners.}
With this definition in place, we can describe the Weighted Multicalibration algorithm and state its guarantees.

\begin{algorithm}
\caption{\label{alg:wmc}Weighted Multicalibration\\
\textbf{Input:}  training data $\set{(x_i,y_i)}_{i=1}^m$\\
Concept class $\mC \subseteq \set{c:\X \to [0,1]}$,\\
Weight class $\mW \subseteq \set{w:\Delta_l \to [0,1]^l}$,\\
approximation $\alpha > 0$,\\
step size $\eta$\\
\textbf{Output:}  $(\mC,\mW,\alpha)$-multicalibrated predictor $f:\X \to \Delta_l$}
\begin{algorithmic}
\STATE $f_0(\cdot) \gets (1/l,\hdots,1/l) \in \Delta_l$
\STATE $mc \gets \mathsf{false}$
\STATE $t \gets 0$
\WHILE{$\neg mc$}
    \STATE $mc \gets \mathsf{true}$
    \FOR{$w \in \mW$}
        \STATE $c_{t+1} \gets \mathsf{WAL}_{\mC,\alpha}\left(\set{\big(x_i, \left\langle w(f_t(x_i)), y_i - f_t(x_i) \right \rangle \big)}_{i=1}^m\right)$
        \IF{$c_{t+1} = \bot$}
            \STATE \textbf{continue}
        \ELSE
            \STATE $\delta_{t+1}(\cdot) \gets w(f_t(\cdot))\cdot c_{t+1}(\cdot)$
            \STATE $f_{t+1}(\cdot) \gets \pi_{\Delta_l}\left(f_t(\cdot) + \eta \cdot \delta_{t+1}(\cdot)\right)$\hfill\textit{// where $\pi_{\Delta_l}$ projects onto $\Delta_l$}
            \STATE $mc \gets \mathsf{false}$
            \STATE $t \gets t+1$
            \STATE \textbf{break}
        \ENDIF
    \ENDFOR
\ENDWHILE
\RETURN $f_t$
\end{algorithmic}
\end{algorithm}

The algorithm is an iterative boosting-style procedure.
We initialize the hypothesis to be the constant function $f_0(\x) = (1/l,\hdots,1/l) \in \Delta_l$.
Then, in the $t$th iteration, for each $w \in \mW$ we reduce the problem of searching for some $c \in \mC$ where $f_t$ is miscalibrated to the problem of weak agnostic learning.
If we find some $c \in \C$ such that
\begin{gather*}
    \E\lr{c(\x) \langle w(f_t(\x)), \y - f_t(\x) \rangle} > \alpha/2
\end{gather*}
then we can use $c(\x)\cdot w(f_t(\x))$ to update the predictor to be better calibrated in this direction.
If for all $w \in \mW$ we fail to find any $c \in \mC$ that correlates with the residual, then we return the current hypothesis.
We describe the procedure in Algorithm~\ref{alg:wmc}.

\paragraph{Analysis of the algorithm.}
The exact running time of the algorithm depends intimately on the model of computation and the time complexity of weak agnostic learning, which for most classes $\mC$ will dominate the time complexity.
With this in mind, we bound the iteration complexity $T$ of the algorithm, noting that each iteration makes at most $\card{\mW}$ calls to $\WAL_{\mC,\alpha}$, which results in a time complexity bounded by $T \cdot \card{\mW}$ times the complexity of weak agnostic learning $\mC$.

First, we argue correctness---that if the algorithm terminates, then the returned hypothesis satisfies multicalibration.
\begin{lemma}
If Algorithm~\ref{alg:wmc} returns a hypothesis $f:\X \to [0,1]^l$, then $f$ is $(\mC,\mW,\alpha)$-multicalibrated.
\end{lemma}
\begin{proof}
Observe that Algorithm~\ref{alg:wmc} only returns a hypothesis $f_t$ if, in the $t$th iteration, for every $w \in \mW$, the call to the weak agnostic learner $\WAL_{\mC,\alpha}$ returns $\bot$.
By the weak agnostic learning property (2), returning $\bot$ in every call indicates that for all $w \in \mW$ and for all $c \in \C$, the correlation between $c$ and the weighted residual is bounded by $\alpha$ in magnitude.
\begin{gather*}
\E\lr{c(\x) \cdot \left\langle w(f_t(\x)), \y - f_t(\x) \right \rangle} \in [-\alpha,\alpha]
\end{gather*}
By definition, this means that $f_t$ is $(\mC,\mW,\alpha)$-multicalibrated.
\end{proof}

Next, we argue that the number of iterations that the algorithm ever runs for is bounded polynomially in $l$ and $1/\alpha$.
\begin{lemma}
\label{lem:iters}
With step-size $\eta = \alpha/2l$, Algorithm~\ref{alg:wmc} returns $f_T$ after $T \le 8l/\alpha^2$ iterations.
\end{lemma}
\begin{proof}
The bound on the number of iterations  follows by a potential argument.
Using the expected squared error as a potential function, we lower bound the progress at each iteration.
Specifically, we use the following potential function,
\begin{gather*}
    \phi(f) = \E\lr{\norm{\fs(\x) - f(\x)}^2}.
\end{gather*}
By the assumption that $\fs:\X \to \Delta_l$ and our choice of $f_0(\x)_i = 1/l$ for all $i \in [l]$, the initial potential value is at most $\phi(f_0) \le 2$.

Consider the change in potential after the $t$th update.
\begin{align*}
\E&\lr{\norm{\fs(\x) - f_t(\x)}^2} - \E\lr{\norm{\fs(\x) - f_{t+1}(\x)}^2}\\
&=\E\lr{\norm{\fs(\x) - f_t(\x)}^2} - \E\lr{\norm{\fs(\x) - f_{t}(\x) - \eta \cdot \delta_{t+1}(\x)}^2}\\
&=2\eta \cdot \E\lr{\left\langle \fs(\x) - f_t(\x), \delta_{t+1}(\x) \right\rangle} - \eta^2 \cdot \E\lr{\norm{\delta_{t+1}(\x)}^2}\\
&\ge 2 \eta \cdot \E\lr{\left\langle \fs(\x) - f_t(\x), w (f_t(\x)) \cdot c_{t+1}(\x)  \right\rangle} - \eta^2 \cdot \E\lr{\norm{w( f_t(\x)) c_t(\x)}} \\
& \geq \alpha \eta - \eta^2 l
\end{align*}
where we use the following bounds 
\begin{itemize}
    \item By our definition of $w$ and $c$,
    \[ \norm{w( f_t(\x)) c_t(\x)}^2 \leq \norm{w_t(f(\x))}^2 \leq l. \]
    \item By the weak agnostic learning property (1), 
\begin{align*}
\E\lr{\left\langle \fs(\x) - f_t(\x), w (f_t(\x)) \cdot c_{t+1}(\x)  \right\rangle}
= \E\lr{c_{t+1}(\x) \cdot \left\langle w(f_t(\x)), \y - f_t(\x)\right \rangle} \ge \alpha/2
\end{align*}
\end{itemize}

Taking $\eta = \alpha/2l$, the progress in $\phi$ in each iteration is at least $\alpha^2/(4l)$.
Since $\phi(f_0) \leq 2$ and $\phi(f_t) \geq 0$ for all $t$, the total number of iterations is upper bounded by $T \le 8l/\alpha^2$.
\end{proof}

We upper bound the sample complexity necessary to run Algorithm~\ref{alg:wmc} in terms of the number of iterations, the cardinality of the weight class $\mW$, and the sample complexity of the weak agnostic learner for $\mC$.
Note that in the $t$th iteration, for each $w \in \mW$, we assign $x_i$ a label that depends on $f_t$.
This dependence on prior hypotheses (and thus prior access to the data), results in an adaptive data analysis problem.
Naively, we can handle this by resampling at each iteration.
We obtain the following generic bound.
\begin{proposition}
\label{prop:samples}
For a hypothesis class $\mC$ and approximation parameter $\alpha_0 > 0$,
the sample complexity $m$ to run Algorithm~\ref{alg:wmc} with success probability at least $1-\beta$ is upper bounded by
\begin{gather*}
    m \le O\left(\frac{l \cdot m(\mC,\alpha_0,\beta_0)}{\alpha_0^2}\right)
\end{gather*}
where $m(\mC,\alpha_0,\beta_0)$ is the sample complexity of running $\WAL_{\mC,\alpha_0}$ with failure probability $\beta_0 \le \frac{\alpha_0^2\beta}{l \cdot \card{\mW}}$.
\end{proposition}
\begin{proof}
The upper bound follows by using a fresh sample for each iteration.
We leverage the upper bound on the number of iterations necessary from Lemma~\ref{lem:iters}, $T \le l/\alpha_0^2$.
Then, to obtain an overall failure probability of $\beta$, we take $\beta_0$ small enough that we can union bound the failure probability of $\WAL_{\mC,\alpha_0}$ over $T \cdot \card{\mW}$ calls.
Again, leveraging the bound on $T$, we bound $\beta_0 \le \frac{\alpha_0^2\beta}{l \cdot \card{\mW}}$.
\end{proof}

Using Proposition~\ref{prop:samples}, we obtain a concrete upper bound on the sample complexity based on specifying a weak agnostic learner and a class of weight functions $\mW$.
For instance, if $\mC$ is a finite class, the weak agnostic learner that iterates over $\mC$ and evaluates the correlation with labels as a statistical query obtains sample complexity $\log(\card{\mC}/\beta_0)/\alpha_0^2$, for an overall sample complexity of
\begin{gather*}
    m \le O\left(\frac{l \cdot \log(l \card{\mC}\card{\mW}/\alpha_0\beta)}{\alpha_0^4}\right).
\end{gather*}
For classes $\mC$ of bounded VC dimension, it is known that the optimal sample complexity of  weak learning is 
\begin{align}
\label{eq:opt-wal}
    m(\mC,\alpha_0,\beta_0) =O\left(\frac{\VC(\mC) + \log(1/\beta_0)}{\alpha_0^2}\right).
\end{align}
Assuming we have access to a weak agnostic learner that has optimal sample complexity, the  sample complexity of Algorithm \ref{alg:wmc} given by Proposition \ref{prop:samples} is
\begin{gather}
\label{eqn:vc:optimal}
    m \le O\left(\frac{l \cdot \left(\mathrm{VC}(\mC) + \log(l\card{\mW}/\alpha_0\beta)\right)}{\alpha_0^4}\right).
\end{gather}

\paragraph{Better sample complexity.}

Improved sample complexity analyses are possible for specialized implementations of the weak agnostic learner.
Following \cite{hkrr2018,KimThesis}, we can avoid some of the cost of resampling by appealing to generalization guarantees for differentially-private learning algorithms \cite{dwork2006calibrating}.
Note that Algorithm~\ref{alg:wmc} only touches the data through the weak agnostic learner, in order to search for a violated constraint for some $c \in \mC$.
By implementing this search step under differential privacy, we can appeal to the results of \cite{dwork2015preserving,bnsssu15,jung2019new}, demonstrating that such algorithms guarantee statistical generalization, even under adaptive access to the data.
For instance, using a bound from Corollary~6.4 of \cite{bnsssu15}, in the case where $\mC$ is a finite class, we can actually bound total the sample complexity as follows.
\begin{gather*}
    m
\le O\left(\frac{l^{1/2}\cdot \log(\card{\mC}\card{\mW}/\alpha)\cdot \log(1/\alpha\beta)^{3/2}}{\alpha^3}\right)
\end{gather*}
This bound follows by viewing each iteration as an optimization over the simultaneous choice of $w \in \mW$ and $c \in \mC$ to maximize the multicalibration violation.
While this approach improves the sample complexity, computationally it requires exhaustive search over the choice of $c \in \mC$ and $w \in \mW$ to execute the exponential mechanism \cite{mcsherry2007mechanism}.

\subsection{Comparing the Sample Complexity Across Notions}

We instantiate the general bound from Proposition~\ref{prop:samples} using the different weight classes corresponding to low-degree, smooth, and full multicalibration. Since there are many parameters involved, to get a fair comparison, we make the following choices:
\begin{enumerate}
\item We instantiate each notion using the accuracy parameter $\alpha_0$ that is required to guarantee $\alpha$-degree-$(k+1)$ multicalibration for some $k \in \N$. In other words, for smooth and full multicalibration, we upper bound the sample complexity using the best choice of $\alpha_0$ known to guarantee that we get $f \in \mcal_{k+1}(\alpha)$.

\item  We assume that we have access to a weak agnostic leaner $\WAL_{\C,\alpha}$ that gives optimal sample complexity (Equation \eqref{eq:opt-wal}), so that the sample complexity of Algorithm \ref{alg:wmc} is bounded by Equation \eqref{eqn:vc:optimal}. Note that an increase in the sample complexity's dependence on $\alpha_0$ will \emph{increase} the gap in sample complexities. Hence, assuming a sample-optimal weak agnostic learner gives a conservative estimate on the gap.
\end{enumerate}

In this setting, we show a substantial gap in the sample complexities of low-degree multicalibration on one hand versus smooth and full multicalibration on the other. The former has sample complexity that grows polynomially with the number of labels $l$, whereas the latter notion have sample complexity that grows exponentially with $l$.

\begin{theorem}[Formal restatement of Theorem~\ref{thm:samples:informal}]
\label{thm:samples}
Suppose $\mC$ has a weak agnostic learner with sample complexity
\begin{gather*}
    m(\mC,\alpha_0,\beta_0) = O\left(\frac{\VC(\mC) + \log(1/\beta_0)}{\alpha_0^2}\right)
\end{gather*}
to obtain desired accuracy $\alpha_0$ over $\mC$ with all but probability $\beta_0$.
Then, for any $k \in \N$ and any failure probability $\beta > 0$, there exists an implementation of Algorithm~\ref{alg:wmc} to obtain the variants of multicalibration, obtaining sample complexity as follows.
\begin{itemize}
    \item \emph{(Low-Degree).}
    \begin{gather*}
        m_k \le O\left(\frac{l \cdot \left(\VC(\mC) + k \cdot \log(l/\alpha\beta)\right)}{\alpha^4}\right)
    \end{gather*}
    to obtain $(\mC,\alpha)$-degree-$(k+1)$ multicalibration.
    \item \emph{(Smooth).}
    \begin{gather*}
        m_s \le O\left(k^4l \cdot \left(\frac{\VC(\mC) + \log(kl/\alpha\beta)}{\alpha^4} + \frac{(kl)^{l-1}\poly(l,\log(k/\alpha))}{\alpha^{l+3}}\right) \right)
    \end{gather*}
    to obtain $(\mC,\alpha/k)$-smooth multicalibration.
    \item \emph{(Full).}
    \begin{gather*}
        m_i \le O\left(\frac{(2kl)^{4(l+1)} \cdot \left(\VC(\mC) + \log(kl/\alpha\beta)\right)}{\alpha^{4(l+1)}}\right)
    \end{gather*}
    to obtain $(\mC,\alpha_0)$-full multicalibration for $\alpha_0 \le (\alpha/2kl)^{l+1}$.
\end{itemize}
In the case of binary prediction, the full multicalibration bound can be improved to
\begin{gather*}
    m_{i,\mathrm{bin}} \le  O\left(\frac{k^4\cdot \left(\VC(\mC) + \log(k/\alpha\beta)\right)}{\alpha^8}\right).
\end{gather*}
\end{theorem}

\begin{proof}
The proof instantiates the bound in (\ref{eqn:vc:optimal}) with an appropriate weight class to achieve the desired notions.

\noindent\emph{(Low-Degree).}  To guarantee that the calibration constraint is satisfied for all $w \in \mW_{k+1}$, we use a discrete set of functions $\mathcal{M}_k$ defined by monomials of degree $\le k$.
In particular, we know that each coordinate of a given $w$ is implemented by some $q \in \mP_k$.
We consider a finite class of functions, where for each $i \in [l]$ and each monomial $s(z) = \prod_{i \in S} z_i$ of degree $\le k$ (where $S$ is a multiset of elements from $[l]$), we include a $1$-sparse function equal to $s(z)$ in the $i$th coordinate and $0$ elsewhere.
\begin{gather*}
    \mathcal{M}_k = \set{s^{(i)} : i \in [l], S \in [l]^n, n \le k}\\
    \textrm{where } s^{(i)}(z)_j = \begin{cases}\prod_{i \in S} z_i & j = i\\
    0 & \textrm{o.w.}\end{cases}
\end{gather*}
For $z \in \Delta_l$, these functions satisfy boundedness and sparsity.
Further by convexity, obtaining $(\mC,\mathcal{M}_k,\alpha)$-multicalibration implies $(\C,\alpha)$-degree-$(k+1)$ multicalibration.
The cardinality of this set $\card{\mathcal{M}_k}$ grows as $O(l^k)$.
We plug this bound on the number of weight functions into the generic sample complexity bound.
\begin{gather*}
    m_{k+1} \le O\left(\frac{l \cdot \left(\VC(\mC) + k \log(l/\alpha\beta)\right)}{\alpha^4}\right)
\end{gather*}

\noindent\emph{(Smooth).}
By Theorem~\ref{thm:deg-to-smooth}, we can take $\alpha_0 \le \alpha/k$ to guarantee a $(\mC,\alpha_0)$-smooth multicalibrated predictor $f$ is also $(\mC,\alpha)$-degree-$(k+1)$ multicalibrated.
Then, for our choice of weight class $\mW$ to guarantee smooth multicalibration, we appeal to Lemma~\ref{lem:smooth:basis:l}, which upper bounds the cardinality $\card{\mW} \le \exp\left(\tilde{O}((l/\alpha_0)^{l-1})\right)$, proved below in Section~\ref{sec:smooth:basis}.
With the choice of $\alpha_0$, we bound the log of this cardinality as follows.
\begin{align*}
\log \card{\mW} \le \tilde{O}\left(\left(\frac{l}{\alpha_0}\right)^{l-1}\right) \le \left(\frac{kl}{\alpha}\right)^{l-1} \cdot \poly(l,\log(k/\alpha))
\end{align*}
Combining these bounds, we can bound the sample complexity for smooth multicalibration as follows.
\begin{align*}
m_s &\le O\left(\frac{k^4l \cdot \left(\VC(\mC) + \log(kl/\alpha\beta) + \log\card{\mW}\right)}{\alpha^4}\right)\\
&\le O\left(k^4l \cdot \left(\frac{\VC(\mC) + \log(kl/\alpha\beta)}{\alpha^4} + \frac{(kl)^{l-1}\poly(l,\log(k/\alpha))}{\alpha^{l+3}}\right) \right)
\end{align*}

\noindent\emph{(Full).}
For $\delta > 0$, by Theorem~\ref{thm:int-to-smooth}, 
$\mcalreg_\delta(\alpha\delta^l/k - l\delta^{l+1}) \subseteq \mcals(\alpha/k) \subseteq \mcal_{k+1}(\alpha)$.
Balancing terms, we take $\delta = \frac{\alpha}{2kl}$, which results in $\alpha_0 \le l^{-l} \left(\frac{\alpha}{2k}\right)^{l+1}$ to guarantee that a $(\mC,\alpha_0,\delta)$-full multicalibrated predictor is also $(\mC,\alpha)$-degree-$(k+1)$ multicalibrated.
The interval basis $\mI_\delta$ has $1/\delta^l$ functions, so we can bound $\log\card{\mI_\delta}$ as follows.
\begin{gather*}
\log\card{\mI_\delta} = l \cdot \log\left({2kl}/{\alpha}\right)
\end{gather*}
With the choice of $\delta$ and $\alpha_0$, we bound the sample complexity as follows.
\begin{align*}
m_i &\le O\left(\frac{l \cdot \left(\VC(\C) + (l+1) \cdot \log(2kl/\alpha\beta) + l \cdot \log(2kl/\alpha)\right)}{l^4(\alpha/2kl)^{4(l+1)}}\right)\\
&\le O\left(\left(2kl\right)^{4(l+1)} \cdot \frac{\VC(\mC) + \log(kl/\alpha\beta)}{l^2\alpha^{4(l+1)}}\right)\\
&\le  O\left(\frac{(2kl)^{4(l+1)} \cdot \left(\VC(\mC) + \log(kl/\alpha\beta)\right)}{\alpha^{4(l+1)}}\right)
\end{align*}
Finally, using the containment of full multicalibration within smooth multicalibration specialized to binary prediction, we can tighten the analysis for $l=2$.
In this case, we can take the binary interval basis $\mI_\delta$ of size $1/\delta$ functions for $\delta = \Theta(\alpha_0^{1/2})$, and applying Proposition~\ref{prop:int-to-smooth}, we can take $\alpha_0 \le O(\alpha^2/k)$, to ensure $(\mC,\alpha/k)$-smooth multicalibration.
In all, we can bound the sample complexity
\begin{gather*}
    m_{i,\mathrm{bin}} \le  O\left(\frac{k^4\cdot \left(\VC(\mC) + \log(k/\alpha\beta)\right)}{\alpha^8}\right),
\end{gather*}
establishing Theorem~\ref{thm:samples}.
\end{proof}

Consistent with the prior results on the relationship between the notions of multicalibration, we see that focusing on low-degree multicalibration can lead to significant sample complexity savings.
In particular, in the multi-class setting, the low-degree complexity provides exponential savings compared to  the smooth and full complexity.
For the binary prediction case, the savings from low-degree multicalibration are only polynomial factors in $\alpha$, but still practically-relevant.
Even for very modest values of $\alpha$, say $0.25$, low-degree multicalibration obtains more than a $200$-fold decrease in sample complexity.

While this analysis doesn't establish lower bounds on the sample complexity, the point is that the savings are coming from the difference in the necessary choice of $\alpha_0$.
Thus, it seems that any sample complexity upper bound should apply equally well for all notions (in terms of $\alpha_0$), will result in an improved complexity for low-degree multicalibration.
Of particular note, a recent work of \cite{gupta2021online} establishes an (inefficient) algorithm with optimal dependence of $\alpha_0^{-2}$ for full multicalibration in the binary prediction case.
Specifically, in our notation, for classes where each group $c \in \mC$ has constant measure in $\mD$, they achieve $(\mC,\alpha_0,\delta)$-full mutlicalibration with probability at least $1-\beta$ in sample complexity that grows as
$$O\left(\frac{\log(\card{\mC}/\delta\beta)}{\alpha_0^2}\right).$$
Setting $\alpha_0$ and $\delta$ to achieve even $(\mC,\alpha)$-multiaccuracy, gives a dependence of $\alpha^{-4}$.
It would be interesting to extend their game-theoretic analysis to low-degree multicalibration, towards obtaining $\alpha^{-2}$ dependence.

\subsection{Better bases for smooth multicalibration in high dimensions}
\label{sec:smooth:basis}

In $l$ dimensions, we can construct a sparse basis at the cost of a much larger sized family of weight functions.

\begin{lemma}
\label{lem:smooth:basis:l}
For any $\eta \in (0, 1)$, there exists a $(\eta, 1)$-basis $\mB_\eta$ for $\mL_{1 \to \infty}$ of size $\exp(O(l\eta^{-1})^{l-1} \log \f1\eta)$.
\end{lemma}
\begin{proof}
We assume that $1/\eta \in \N$ is an integer. It will be enough to show such a basis for $\mL_1$, since we can handle each coordinate of the output separately, at the cost of a multiplicative factor of $l$ on the size of the family. 

Break up $[0, 1]^{l-1}$ into $(3l/\eta)^{l-1}$ cubes of side length $\eta/3l$ each. Then, let $\mB_\eta$ be the all functions on $\Delta_l$ which take one of the constant values $0, \eta/3, 2\eta/3, \dots, 1$ on each of these cubes (where we have projected away the last coordinate $x_l$ in $\Delta_l$). We have
\[
|\mB_\eta| \leq (3/\eta+1)^{(3l/\eta)^{l-1}} = \exp\left(O(l\eta^{-1})^{l-1} \log \f1\eta\right) \] 

Now, we claim that every $u \in \mL_1$ can be approximated to within $\eta$ by a single function in $\mB_\eta$. Indeed, for each of the cubes, round $u$ down to the nearest multiple $\eta/3$ on one of the corners of the cube. Construct the function $v \in \mB_\eta$ by letting $v$ take on this rounded value on that cube. Then, $v$ will be within $\eta/3$ of $u$ on that corner of the cube.

We claim that the whole cube (projected up to $\Delta_l$) is within distance $2\eta/3$ in $L_1$ from the corner. To see this, note that by construction, the distance in $L_1$ in the first $l-1$ coordinates is at most $(l-1) \cdot \eta/3l < \eta/3$. Also, in $\Delta_l$, the distance in the last coordinate is at most the $L_1$ distance in the rest of the coordinates, so this is at most $\eta/3$ as well.

Thus the whole cube is within $2\eta/3$ in $L_1$ from the corner, so since $u \in \mL_1$, this means that the value of $u$ on the whole cube is within $2\eta/3$ of its value of the corner. Therefore, the value of $u$ on the whole cube is within $\eta/3 + 2\eta/3 = \eta$ of the value of $v$ on the cube. Thus, $\infnorm{u -v} \leq \eta$.
\end{proof}

\section{Experiments}
\label{sec:experiments}

We use a numerical experiment to compare boosting for multiaccuracy (MA), degree-$2$ multicalibration (MC2) and full multicalibration (MC-full).
The goal of this experiment is two-fold.
First, the theoretical sample complexity results are asymptotic; here, we show that qualitatively similar results hold empirically in the finite sample regime.
Second, the sandwiching bounds show that MC2 can reduce overconfidence, as compared to multiaccurate predictors;
we supplement the theory by showing a setting in which multiaccuracy post-processing does not correct for initial overconfidence, but degree-2 multicalibration does.
In combination, these preliminary experiments suggest that the strongest notion of multicalibration is not always better.
Given a fixed data set size, the realized fairness guarantees may actually improve by choosing a lower degree of multicalibration.

\paragraph{Metrics.}
We measure the performance of predictors across the first two moments across subpopulations $c \in \mC$.
Specifically, we measure the multiaccuracy error as
\begin{align}
    \text{\textbf{multiaccuracy error}:} & & \max_{c \in \mC} \E[c(\x) (f(\x) - f^*(\x)) ] - \E[(1-c(\x)) (f(\x) - f^*(\x)) ] 
\end{align}
Intuitively, this is the multi-accuracy error on the subpopulation and its complement. 
Second, we measure the excess variance as
\begin{align}
     \text{\textbf{excess variance}:  } & & \max_{c \in \mathcal{C}} \left( \Var[f(\x)  \mid c(\x)=1] - \Var[ f^*(\x) \mid c(\x)=1] \right)\cdot \Pr[c(\x) = 1]
\end{align}
which intuitively is how much the variance of the predicted probability exceeds the variance of the optimal predictions over all subpopulations in $\mC$.

\paragraph{Setup.}  To estimate the excess variance we need access to the true probability $f^*$ which is unavailable for real datasets. Therefore, we use a semi-synthetic dataset by fitting a neural network to the real UCI-adult dataset and use the neural network's predicted probability as the ``true'' Bayes optimal probability. We also fit a generative model (a variational autoencoder) to model the distribution on the features $\x$. Combined we create a synthetic dataset where we can sample $\x$ from the generative model, compute the ``true'' probability $f^*(\x)$, and draw samples $\y$ from the ``true'' probability. Note that all learning algorithm only have access to the samples $(\x, \y)$ and not the ``true'' probability; we use the ``true'' probability exclusively for computing the excess variance. 

We generate three datasets: a pre-training set, a training set, and a test set.
We first pretrain a three-layer neural network on the pre-training set, then use our boosting algorithms to adjust the predictions of the pretrained neural network to achieve multi-accuracy or calibration, and finally use the test set to assess performance.
For the calibration class $\mathcal{C}$, we use linear functions with sigmoid activation.

\paragraph{Results.} Our results are summarized in Figure~\ref{fig:uci_adult}.
We make the following observations.
\begin{itemize}
    \item The boosting-style algorithm for multi-accuracy (MA), degree 2 multi-calibration (MC2) and full multicalibration (MC-full) all improve the multi-accuracy error on the training set.
    This is consistent with our result that MC2 and MC-full imply MA, hence by achieving MC2 and MC-full we can also achieve multiaccuracy (MA).
    However, on the test set, we observe that multicalibration is much more prone to overfitting and the multi-accuracy error increases rapidly without carefully regularization (e.g. by early stopping).
    \item Boosting algorithms for degree-2 multicalibration (MC2) and full multicalibration (MC-full) can significantly decrease the excess variance.
    However, degree-2 multi-calibration is much less prone to overfitting and can consistently keep the excess variance low, while multicalibration rapidly overfits.
\end{itemize}

Overall we observe that running Algorithm~\ref{alg:wmc} for degree-2 multicalibration (MC2) can reduce the excess variance without harming the multiaccuracy error, and generally maintains the generalization performance as compared to multiaccuracy (MA) only.
On the other hand, boosting for multicalibration (MC-full) is significantly more prone to overfitting.

\begin{figure}
    \centering
    \includegraphics[width=1.0\linewidth]{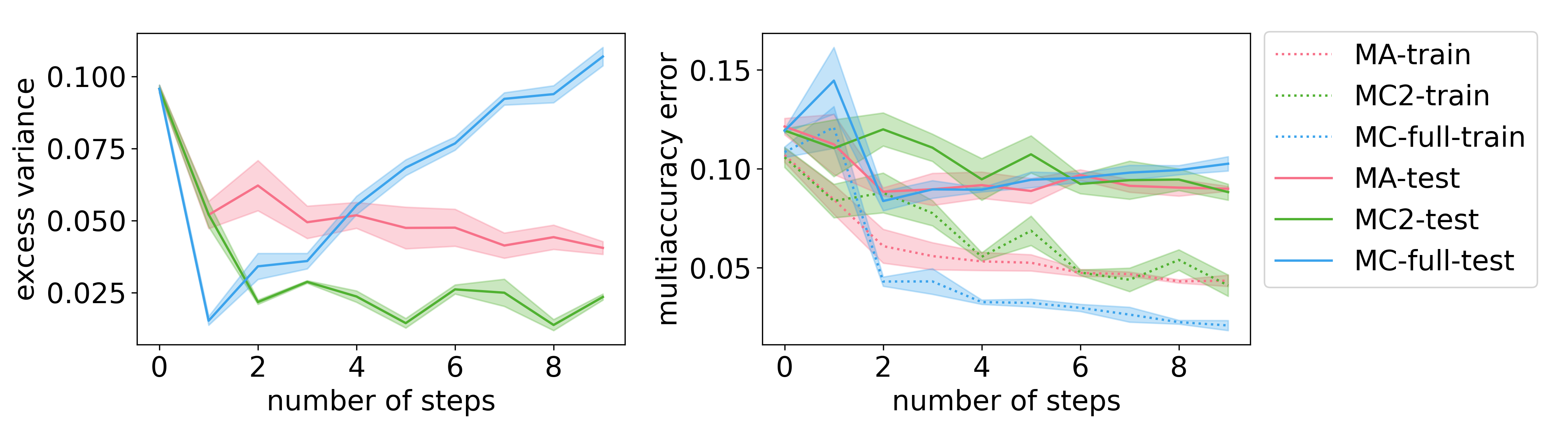}
    \caption{Comparing the excess variance and multiaccuracy error for three methods: boosting for multiaccuracy (MA), degree-2 multicalibration (MC2) and full multicalibration (MC-full). Error bars are 1 standard deviation of the results based on 5 randomly drawn datasets. Both MA and MC2 achieve low multiaccuracy error, but MC2 has much lower excess variance, consistent with our theoretical results. MC-full is very prone to overfitting.}
    \label{fig:uci_adult}
\end{figure}

\paragraph{Acknowledgments.}
The authors thank Gal Yona for useful feedback on an earlier draft of this manuscript and Omer Reingold and Udi Wieder for helpful discussions. We are grateful to Pranay Tankala who alerted us to an error in our analysis of Lemma \ref{lem:iters} in an earlier version of this paper. 

\clearpage
\bibliographystyle{alpha}
\bibliography{refs}

\clearpage
\appendix

\section{Squared loss minimization from degree-2 multicalibration}
\label{app:lossmin}
We give a proof for the claim of squared loss minimization given in the technical overview.
\begin{proposition}
Suppose $f:\X \to [0,1]$ is $(\C,0)$-degree-$2$ multicalibrated for some class $\C$ that contains the constant function $c(\x) = 1$.
Then, for all $c \in \C$
\begin{gather*}
    \E\lr{(c(\x)-\fs(\x))^2} \ge \E\lr{(f(\x) - \fs(\x))^2}
\end{gather*}
In fact the following Pythagorean bound holds:
\begin{gather*}
    \E\lr{(c(\x)-\fs(\x))^2} = \E\lr{(f(\x) - \fs(\x))^2} + \E[\lr{(c(\x) - f(\x))^2}
\end{gather*}
\end{proposition}
\begin{proof}
It suffices to prove the Pythagorean bound, the inequality follows from it.
For $c \in \C$, we consider the difference in squared error with $f$:
\begin{gather*}
\E\lr{(c(\x) - \fs(\x))^2} - \E\lr{(f(\x) - \fs(\x))^2}
= \E\lr{c(\x)^2 - f(\x)^2 - 2c(\x)\fs(\x) + 2f(\x)\fs(x)}
\end{gather*}
By first moment equality, we have that
\begin{gather*}
    \E\lr{c(\x) \fs(\x)} = \E\lr{c(\x) f(\x)}
\end{gather*}
and by degree-$2$ calibration, we have that
\begin{gather*}
    \E\lr{f(\x) \fs(\x)} = \E\lr{f(\x)^2}.
\end{gather*}
Thus, in all, we can simplify the expression as follows.
\begin{align*}
&\E\lr{c(\x)^2 - f(\x)^2 - 2c(\x)\fs(\x) + 2f(\x)\fs(\x)}\\
&=\E\lr{c(\x)^2 + f(\x)^2 - 2c(\x)f(\x) + 2f(\x)\fs(\x) - 2f(\x)^2}
\\
&= \E\lr{(c(\x) - f(\x))^2} \ge 0
\end{align*}
\end{proof}
 \section{Multiaccuracy does not imply positive correlation}

\label{sec:correlation-counterexample}

Note that many common regression methods do not in fact guarantee a positive or near-positive correlation between $f$ and $y$, conditioned on a constraint $c$. We provide an example in the case of binary classification illustrating this.

\begin{example} \label{ex:l2-logistic}
Let $\mathcal X = \{0, 1\}^2$; we will refer to the two coordinates as $x_1$ and $x_2$. Let the four constraints $c_{i, b}$, for $i \in \{1, 2\}$ and $b \in \{0, 1\}$ be 1 exactly when $x_i = b$, and 0 elsewhere. Let the distribution $\mD$ have weight $1/3$ on $(0, 0)$ and $(0, 1)$, and weight $1/6$ on $(1, 0)$ and $(1, 1)$. Finally, let the value $\y$ be always equal to the parity of $\x$; that is, $\y = x_1 \oplus x_2$. \cref{fig:l2-logistic-fig} illustrates this example.

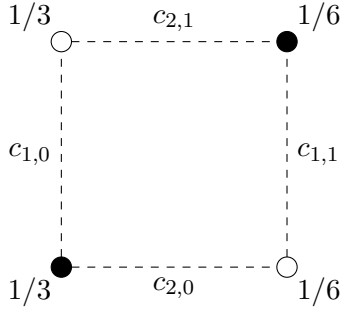
\begin{figure}
\centering
\begin{tikzpicture}[scale=1.5]
\def\r{0.09}
\filldraw[black] (0,0) circle (\r) node[below left]{$1/3$};
\draw[black] (2,0) circle (\r) node[below right]{$1/6$};
\draw[black] (0,2) circle (\r) node[above left]{$1/3$};
\filldraw[black] (2,2) circle (\r) node[above right]{$1/6$};

\draw[dashed] (\r,0)--(2-\r,0);
\draw[dashed] (2,\r)--(2, 2-\r);
\draw[dashed] (0,\r)--(0,2-\r);
\draw[dashed] (\r,2)--(2-\r,2);

\node[below] at (1,0) {$c_{2,0}$};
\node[above] at (1,2) {$c_{2,1}$};
\node[left] at (0,1) {$c_{1,0}$};
\node[right] at (2,1) {$c_{1,1}$};

\end{tikzpicture}
\caption{Diagram illustrating \cref{ex:l2-logistic}. Black points indicate points $\x$ such that $\y=1$ always, and white points indicate where $\y=0$ always. The points are labeled with their probability under $\mD$. The constraints each contain two points, and are drawn as edges.}
\label{fig:l2-logistic-fig}
\end{figure}

\end{example}

We will show that, in \cref{ex:l2-logistic}, both $L_2$ and logistic regression obtain predictions $f$ which have negative correlation with $y$ conditioned on $c_{1,1}$.

First, \textit{$L_2$ regression} finds the predictor $f$ of the form
\[f(\x) = \sum_{i,b} \lambda_{i,b} c_{i,b}(\x),\]
such that the objective
\[\E_{\mD} [(f(\x)-\y)^2]\]
is minimized. By first order optimality, $f$ is actually a multiaccurate predictor. We can see that if we apply $L_2$ regression, the obtained coefficients and predictor $f$ are:
\begin{equation*}
\lambda_{1,0}=\lambda_{1,1}=1/2, \lambda_{2,1}=1/6, \lambda_{2,0}=-1/6,
\end{equation*}
\begin{equation} \label{eq:l2-prediction}
f((0,0))=f((1,0))=2/3, f((0,1))=f((1,1))=1/3.
\end{equation}
This can be seen by running $L_2$ regression, or as follows. First, note that the constraints $c_{i,b}$ are linearly dependent (satisfying $c_{1,0}+c_{1,1}=c_{2,1}+c_{2,0}$), so we may assume that $\lambda_{1,0}=1/2$. The remaining constraints are now linearly independent, so since the objective is strictly convex, there is a unique optimal choice of $\lambda_{i,b}$. Now note that the example exhibits symmetry by exchanging $x_2=0$ and $x_2=1$, and flipping each of the $\y$ values. Thus, the value of the objective is conserved under the substitutions \[\lambda_{1,1} \leftarrow 1-\lambda_{1,1},\; \lambda_{2,0} \leftarrow -\lambda_{2,1},\; \lambda_{2,1} \leftarrow -\lambda_{2,0}.\]
But since the optimum is unique, this implies that $\lambda_{1,1}=1/2$ and $\lambda_{2,0} = -\lambda_{2,1}$. Finally, we can obtain the actual value of $\lambda_{2,0}$ by using multiaccuracy, or by directly optimizing the objective. 
Finally, with the predicted values (\ref{eq:l2-prediction}) we obtain the covariance conditioned on $c=c_{1,1}$ equal to
\[\Cov_{\mD_{c}}[f(\x), \y] = -1/12,\]
showing that $f$ is negatively correlated with $y$ conditioned on $c_{1,1}$.

Next, \textit{logistic regression} finds the predictor 
\[h(\x) = \frac{1}{1+\exp\p{-\sum_{i,b} \theta_{i,b} c_{i,b}(\x)}},\]
maximizing the objective
\[\E_{\mD} [\y \log h(\x) + (1-\y) \log(1-h(\x))].\]
Again, by first-order optimality $h$ is also a multiaccurate predictor. This time, the obtained coefficients and predictor are:
\begin{equation*}
\theta_{1,0}=\theta_{1,1}=0, \theta_{2,1}=\log 2, \theta_{2,0}=-\log 2,
\end{equation*}
\begin{equation}
h((0,0))=h((1,0))=2/3, h((0,1))=h((1,1))=1/3.
\end{equation}
This can again be seen by essentially an identical argument as for $L_2$ regression.
Note that this is also the exact same predictor as that of $L_2$, so we obtain a similar negative correlation.

Note that these examples can also be modified by changing the probabilities under the distribution $\mD$ from $1/3$ and $1/6$ to $1/2-\eps$ and $\eps$, respectively, as $\eps$ gets arbitrarily small. The same argument shows that $f((1,0))$ gets arbitrarily close to $1$ while $f((1,1))$ gets arbitrarily close to $0$. This achieves covariance arbitrarily close to $-1/4$, which is the lowest possible covariance between $[0, 1]$ random variables. This is at the cost of $\mD(c_{1,1})=2\eps$ getting arbitrarily small, so the constraint that we condition on gets arbitrarily low in probability, making statements about the constraint less meaningful.
 
\end{document}